\documentclass[lettersize,journal]{IEEEtran}
\usepackage{cite}
\usepackage{amsmath,amssymb,amsfonts}
\usepackage{algorithmic}
\usepackage{graphicx}
\usepackage{textcomp}
\usepackage{bm}

\usepackage[utf8]{inputenc}
\usepackage{caption}
\usepackage{amsthm}
\usepackage{booktabs}
\usepackage{algorithm}
\usepackage{url}
\usepackage{color}

\usepackage[hidelinks]{hyperref}
\usepackage{balance} 
\usepackage{latexsym}
\usepackage{enumitem}
\usepackage{subcaption}
\usepackage{anyfontsize}
\usepackage{glossaries}
\glsdisablehyper
\usepackage{multirow}
\usepackage{nicefrac}
\newcommand{\etal}{\emph{et al}.}
\usepackage{pifont}

\usepackage{adjustbox}
\usepackage{tabularx}
\newcolumntype{Y}{>{\centering\arraybackslash}X}
\usepackage{newfloat}
\usepackage{listings}

\floatstyle{ruled}
\newfloat{listing}{tb}{lst}{}
\floatname{listing}{Listing}
\DeclareMathOperator*{\argmin}{arg\,min}

\newtheorem{definition}{Definition}[section]
\newtheorem{example}{Example}[section]
\newtheorem{theorem}{Theorem}[section]

\newtheorem{remark}{Remark}

\newacronym{wsu}{WSU}{\emph{Western Sydney University}}
\newacronym{CP}{CP}{\emph{Career Planning}}
\newacronym{cacgs}{CACGs}{\emph{Computer-Assisted Career Guidance Systems}}
\newacronym{AI}{AI}{Artificial Intelligence}
\newacronym{dpp}{DPP}{\emph{Degree Planning Problem}}
\newacronym{qp}{QP}{\emph{Quadratic Programming}}
\newacronym{ip}{IP}{\emph{Integer Programming}}
\newacronym{mip}{MIP}{\emph{Mixed-Integer Programming}}
\newacronym{micp}{MICP}{\emph{Mixed-Integer Convex Programming}}
\newacronym{miqp}{MIQP}{\emph{Mixed-Integer Quadratic Programming}}
\newacronym{miqcp}{MIQCP}{\emph{Mixed-Integer Quadratically Constrained Programming}}
\newacronym{milp}{MILP}{\emph{Mixed-Integer Linear Programming}}
\newacronym{ilp}{ILP}{\emph{Integer Linear Programming}}
\newacronym{uctp}{UCTP}{\emph{University-Course Timetabling Problem}}
\newacronym{csp}{CSP}{\emph{Constraint Satisfaction Problem}}
\newacronym{rcpsp}{RCPSP}{\emph{Resource-Constraint Project Scheduling Problem}}
\newacronym{rpp}{RPP}{\emph{Research Project Planning}}
\newacronym{jssp}{JSSP}{\emph{Job-Shop Scheduling Problem}}
\newacronym{wap}{WAP}{\emph{Workforce Assignment Problem}}
\newacronym{ga}{GA}{\emph{Genetic Algorithm}}
\newacronym{sa}{SA}{\emph{Simulated Annealing}}
\newacronym{ts}{TS}{\emph{Tabu Search}}
\newacronym{pso}{PSO}{\emph{Particle Swarm Optimization}}
\newacronym{ls}{LS}{\emph{Least Squares}}
\newacronym{lp}{LP}{\emph{Linear Programming}}
\newacronym{cp}{CP}{\emph{Constraint Programming}}
\newacronym{socp}{SOCP}{\emph{Second-Order Cone Programming}}
\newacronym{sfia}{SFIA}{\emph{Skills Framework for the Information Age}}
\newacronym{ers}{EdRecSys}{\emph{Educational Recommender Systems}}
\newacronym{er}{EdRecSys}{\emph{Educational Recommender System}}
\newacronym{rs}{RecSys}{\emph{Recommender Systems}}
\newacronym{iis}{IIS}{\emph{Irreducible Infeasible Subsystem}}
\newacronym{mus}{MUS}{\emph{Minimal Unsatisfiable Subset}}
\newacronym{mss}{MUS}{\emph{Maximal Satisfiable Subset}}
\newacronym{mcs}{MCS}{\emph{Minimal Correction Subset}}
\newacronym{mfs}{MFS}{\emph{Maximal Feasible Subsystem}}
\newacronym{tel}{TEL}{\emph{Technology Enhanced Learning}}
\newacronym{ai}{AI}{\emph{Artificial Intelligence}}
\newacronym{xai}{XAI}{\emph{eXplainable Artificial Intelligence}}
\newacronym{sdg}{SDG}{\emph{Sustainable Development Goals}}
\newacronym{gfi}{GFI}{\emph{Gunning Fog Index}}
\newacronym{exmip}{LLMxMIP}{\emph{Large Language Models for Explainable MIP}}
\newacronym{asp}{UDSP}{\emph{User-Desired Satisfiability Problem}}
\newacronym{ml}{ML}{\emph{Machine Learning}}
\newacronym{ca}{CA}{\emph{Combinatorial Auction}}
\newacronym{cap}{CAP}{\emph{Combinatorial Auction Problem}}
\newacronym{wsp}{WSP}{\emph{Weighted Set Packing}}
\newacronym{wdp}{WDP}{\emph{Winner Determination Problem}}
\newacronym{psplib}{PSPLIB}{\emph{Project Scheduling Problem Library}}
\newacronym{cats}{CATS}{\emph{Combinatorial Auction Test Suite}}
\newacronym{llm}{LLM}{\emph{Large Language Model}}
\newacronym{nl}{NL}{\emph{Natural Language}}

\newacronym{bfs}{BFS}{\emph{Breadth-First Search}}

\newacronym{dfs}{DFS}{\emph{Depth-First Search}}
\newacronym{dss}{DSS}{\emph{Decision Support Systems}}

\newcommand{\roger}[1]{{\color{blue}#1}}

\newcommand{\filippo}[1]{{\color{brown}#1}}

\definecolor{olive}{rgb}{0.5,0.5,0}
\definecolor{orange}{rgb}{1.0,0.65,0}
\definecolor{brown}{rgb}{0.65,0.16,0.16}

\begin{document}

\title{Constraint Reasoning and Large Language Models \\ Join Forces: Providing Contrastive Explanations for \\ Real-World Optimisation-Based Applications}


\author{
\IEEEauthorblockN{Roger X. Lera-Leri\IEEEauthorrefmark{1}\IEEEauthorrefmark{3},
Filippo Bistaffa\IEEEauthorrefmark{1},
Athina Georgara\IEEEauthorrefmark{2},
Juan Antonio Rodríguez-Aguilar\IEEEauthorrefmark{1},}

\IEEEauthorblockA{\IEEEauthorrefmark{1}Artificial Intelligence Research Institute, IIIA-CSIC},
\IEEEauthorblockA{\IEEEauthorrefmark{2}University of Southampton}
\thanks{\IEEEauthorrefmark{3} Corresponding author: Roger X. Lera-Leri (rlera@iiia.csic.es)}
}

\markboth{ArXiv}%
{Lera-Leri \MakeLowercase{\textit{et al.}}: Providing Contrastive Explanations for Optimisation-Based Applications}


\maketitle

\begin{abstract}
We focus on optimisation-based \textit{Decision Support Systems} in real-world applications. In this context, we aim to provide \textit{contrastive explanations} that answer users’ questions about recommendations based on the solution of an optimisation problem, specifically a \textit{Mixed-Integer Programming} (MIP). Addressing this endeavour is crucial to ensure human oversight of the decisions implemented, as required by international guidelines on trustworthy and responsible automated decision-making systems. We propose a novel methodology that combines constraint reasoning techniques and \textit{Large Language Models} (LLMs) to generate contrastive explanations for optimisation problems encoded as MIPs. Although LLMs have shown impressive capabilities in Natural Language processing, they can produce incorrect answers due to \textit{reasoning} limitations. To overcome these limitations, we leverage constraint reasoning techniques to distil the information necessary to produce an explanation, which is then passed as input to an LLM as the so-called \textit{graph of reasons}. Such a graph captures the semantics and the relationships among constraints of the optimisation problem. Results on real-world scheduling domains show that incorporating constraint reasoning improves explanations generated by the LLM. In particular, using the graph of reasons enables the LLM to produce natural-language explanations that include more relevant reasons while reducing errors compared to explanations generated without this structure.
\end{abstract}

\begin{IEEEkeywords}
Contrastive Explanations, Constraint Reasoning, Large Language Models, Optimisation-based Decision Support Systems.
\end{IEEEkeywords}

\section{Introduction}
\label{sec:introduction}
\glsresetall

\IEEEPARstart{I}{n} many real-world domains, decision-making processes are driven by \gls{dss} based on recommendations computed by an optimisation approach (i.e., the solution to an optimisation problem).
For instance, scheduling optimisation problems have been used to guide human decisions across various application domains, such as workforce scheduling \cite{bocewicz2023preventive}, food industry production \cite{wauters2012real}, and construction projects \cite{mirnezami2023integrated}. 
In this work, we consider two real-world scheduling problems as our test cases, both formulated as \gls{mip} models.
First, we consider the well-known \gls{rcpsp} \cite{hou2024optimize,mitsou2026ai,chen2022filtering}, a general framework that, over the years, has been employed to model a wide variety of real-world applications~\cite {hartmann2010,Burke01122006,Yousif15122021}. 
Secondly, we consider a problem in the educational domain, the \gls{dpp} \cite{lera2025computing}. This problem aims to compute a personalised bachelor's degree plan---a sequence of subjects across semesters ---that allows a student to acquire the professional skills required for their preferred job. 
Planning which subjects to undertake is a challenging and important task that significantly impacts students' professional careers.
Notice that, in contrast with the \gls{rcpsp}, which aims at minimising the makespan required to complete all activities, the focus of the \gls{dpp} is to maximise the preferences of the users (represented by their desired jobs), hence representing a complementary use case with respect to \gls{rcpsp}. 

In these scenarios, an optimisation algorithm would compute and recommend to a user a project scheduling or a bachelor's degree plan. Besides computing a solution to the optimisation problem at hand, there is the challenge of \emph{explaining} the solution to an end user. Recently, this challenge has also been identified as a key issue by AI researchers. Thus, the need to explain the decisions of a DSS has recently received much attention in the AI community, particularly because of the widespread use of powerful ``\textit{black-box}'' \gls{ml} models \cite{adadi2018peeking}. Such interest has been amplified by regulations that establish individuals’ right to receive meaningful justifications for algorithmic decisions \cite{URL-EU-GDPR}. However, despite its actual-world relevance, optimisation-based DSS has received much less attention.

In this context, this paper focuses on harnessing AI techniques to generate explanations for optimisation-based \gls{dss}. More precisely, our research question is how to leverage the capabilities of \glspl{llm} to 
answer the questions that users might pose to the \gls{dss} to \emph{challenge} its recommended decisions (e.g., ``Why do you recommend taking a course on \emph{Cybersecurity} instead of a course on \emph{Computational Finance}?''). These types of questions reflect a human tendency to seek \emph{contrastive} explanations ---those that do not merely state why something was recommended, but why it was recommended rather than an alternative (e.g.,``\textit{Computational Finance is not recommended because it is an elective course whose skills are not relevant to your desired job, whereas Cybersecurity contributes to acquiring a high level of Vulnerability Research, a skill required for your desired job.}''). 
The usefulness of contrastive explanations is supported by the social sciences, which argue that explanations are most useful and satisfying when they are contrastive \cite{miller2019explanation}. Notice that aiming at contrastive explanations aligns with state-of-the-art works in \gls{xai}, which pursue generating contrastive explanations to justify the outcomes of \gls{ml} algorithms \cite{stepin2021survey,miller2021contrastive,ignatiev2021contrastive}.

Despite their remarkable capabilities in natural language processing, \glspl{llm} are known to struggle with abstract reasoning tasks~\cite{gendron2023large}. This limitation becomes particularly problematic when attempting to explain the outcome of optimisation problems, whose explanations typically depend on intricate reasoning across many interdependent constraints. As a result, directly prompting an \gls{llm} to explain an outcome (as done, for instance, by Rajani \etal~\cite{rajani2019explain}) may lead to unreliable explanations~\cite{lin2024towards,rehman2024hallucination}. In particular, \glspl{llm} may produce \emph{incorrect} explanations, for example, by asserting that a user-proposed alternative is infeasible when it is not (as discussed in Section \ref{subsec:text}). Moreover, their explanations are often \emph{uninformative}, as they tend to offer vague statements about the existence of constraints without explicitly identifying which constraints prevent the user-proposed alternative outcome. 

\begin{figure*}[t]
        \centering
    \includegraphics[width=0.9\textwidth]{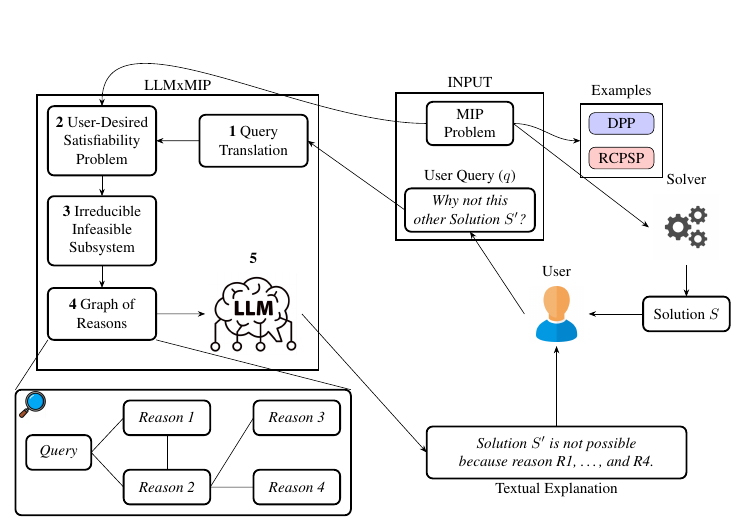}
    \caption{Overview of our explanatory approach. Section \ref{sec:expl} discusses its steps in detail.}
\label{fig:overview}
\end{figure*}

To overcome this lack of abstract reasoning capabilities of \glspl{llm}, in this paper, we propose \gls{exmip}, a novel methodology that places constraint reasoning at the service of \glspl{llm} to generate textual explanations that are \emph{correct} and 
\emph{informative}.
More specifically, \gls{exmip} is a process that follows five main steps, outlined in Figure \ref{fig:overview}.
First, it transforms a user's query, i.e., the question posed by the user, into a constraint and reformulates the original optimisation problem as a \gls{csp} that is compliant with the user’s query (first and second steps, respectively). 
Then it identifies conflicts that prevent us from obtaining a better solution that accommodates the user query. This amounts to extracting the set of constraints in the \gls{csp} that make it infeasible (third step).\footnote{The infeasibility of the \gls{csp} proves that the query of the user leads to a solution that is worse than the computed optimal one.} Our approach exploits constraint reasoning to compute such a set of constraints, i.e., the minimal subset of constraints that make the \gls{csp} infeasible, namely the \gls{iis} \cite{chinneck2007feasibility}.
Afterwards, \gls{exmip} transforms the unstructured collection of mathematical constraints in the \gls{iis} into a structured graph, the so-called \emph{graph of reasons}. 
Such a graph captures the relationships between constraints and assigns them a semantic interpretation --- which we call ``\emph{reason}''--- by means of \gls{nl} templates. 
Finally, an \gls{llm} leverages the graph of reasons to generate a textual explanation. 

Overall, we significantly extend prior work \cite{leraleri2025exploiting} on explainability for \acrshort{milp}s, which relies solely on constraint reasoning, by introducing the following novel contributions:

\begin{itemize}
    \item We propose \gls{exmip}, a novel methodology that integrates state-of-the-art \gls{ai} techniques ---combining constraint reasoning with \glspl{llm}--- to explain the outcomes of an optimisation-based process encoded as a \gls{mip}. 
    \item We employ constraint reasoning techniques to identify the underlying reasons that explain why it is not possible to obtain a better solution that satisfies the user’s query. 
    Then we transform this unstructured set of reasons into a graph that captures their relationships. 
    \item In contrast to \cite{leraleri2025exploiting}, 
    we leverage \glspl{llm} to generate textual explanations from the graph of reasons, whose structured representation enables \glspl{llm} to produce complete and accurate explanations.
    \item We test \gls{exmip} on two prominent real-world optimisation-based scenarios, i.e., the \gls{dpp} (considering the data from a real Bachelor's Degree \cite{lera2025computing}) 
    and the \gls{rcpsp}. 
    Our results indicate that using graphs of reasons as input to \glspl{llm} results in textual explanations that are significantly better in terms of completeness and readability. Additionally, we show that computing explanations with \gls{exmip} is faster than solving such problems in most cases. 
\end{itemize}

This paper is organised as follows. Section \ref{sec:back} discusses background concepts of integer programming and constraint programming, and the basic notions of the scheduling problems used in the experimental evaluation. In Section \ref{sec:expl}, we propose our approach.
In Section \ref{sec:exp}, we evaluate the quality of the explanations generated with \glspl{llm}. Additionally, we perform an empirical analysis of the computational cost of providing explanations. We discuss the related work in Section \ref{sec:related}. Section \ref{sec:conclusion} concludes the paper and discusses future research directions.

\section{Background}
\label{sec:back}

\subsection{Mathematical Optimisation}\label{subsec:mip}

As previously mentioned, this paper proposes a general methodology to explain the recommendations of optimisation-based \gls{dss}, i.e., systems whose recommendations or decisions are computed by solving an optimisation problem. 
To this end, we first introduce basic concepts of mathematical optimisation. 

Mathematical optimisation focuses on finding the best possible solution to a problem, possibly subject to a set of constraints. A solution corresponds to a numerical assignment to a set of decision variables, where these variables encode concrete real-world decisions. An optimisation problem therefore seeks an assignment that (i) minimises or maximises a given objective function and (ii) satisfies all imposed constraints. 
For example, solving a scheduling problem usually amounts to finding the schedule of tasks that respects the particular schedule restrictions and minimises the total activity duration. 

A well-established technique to formalise optimisation problems is \acrfull{mip}. \gls{mip} problems are those that have discrete and continuous decision variables \cite{wolsey2020integer}. In simple terms, the outcome of \gls{mip} problems includes some variables that take fractions or real values and others that are restricted to integers. 
\gls{mip} represents a general class of optimisation problems. Indeed, \gls{ip} is a special case of \gls{mip} in which all variables are restricted to integer values \cite{wolsey1999integer}. Conversely, continuous optimisation is another special case of \gls{mip} in which all variables are continuous.
In general, \gls{mip} problems (with integer variables) are NP-Hard \cite{wolsey2020integer} and are expensive to solve.

In more detail, an \gls{mip} is an optimisation problem in the form of
\begin{subequations}\label{prob:mip:intro}
\begin{align}
    \mbox{minimise} \quad & f_0(\textbf{x},\textbf{y}) \label{prob:mip:o:intro} \\
    \mbox{subject to} \quad & f_i(\textbf{x},\textbf{y})\leq 0 , \quad \forall i = 1,\dots, m, \label{prob:mip:c1:intro}    
\end{align}
\end{subequations}
where $\textbf{x}\in \mathbb{R}^{n}$ is a vector of continuous decision variables, $\textbf{y}\in \mathbb{Z}^{k}$ is a vector of integer decision variables, and $f_i:\mathbb{R}^{n} \times \mathbb{Z}^{k} \xrightarrow{} \mathbb{R}$ is any real-valued function. Notice that we do not assume any form for functions $f_i$. More specifically, the family of MIPs where the objective function $f_0$ is linear or quadratic is characterised as \gls{milp} and \gls{miqp} respectively. 


\begin{figure}
    \centering
    \includegraphics[width=0.7\columnwidth]{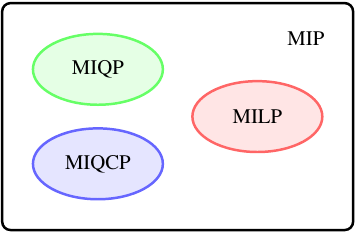}
    \caption{Classes of MIP problems. While MILPs consider linear objective and constraints, MIQP considers a quadratic objective function, and MIQCP comprises problems with quadratic constraints. }
    \label{fig:mip_classes:intro}
\end{figure}


\begin{remark}
    Since we aim to explain the outcome of a given optimisation problem, we assume that such a solution can be computed in the first place by means of an optimisation solver within a reasonable amount of time. Along these lines, with a small abuse of terminology, in the remainder of this paper, we will refer to the families of optimisation problems that a given solver can solve as ''\emph{optimisation problems}''. More specifically, in this paper, we employ CPLEX \cite{cplex2009v12}, which can solve the families of optimisation problems shown in Figure \ref{fig:mip_classes:intro}, i.e., \gls{milp}, \gls{miqp}, and \gls{miqcp}. We adopt this approach since the explanation technique proposed in this work is solver-agnostic, i.e., it can be applied to any optimisation problem, provided that a solver capable of computing a solution exists.\footnote{For instance, our technique can be directly extended to \gls{micp} (i.e., a family that comprises and extends the three subclasses shown in Figure \ref{fig:mip_classes:intro}) by switching to the MOSEK solver \cite{aps2019mosek}, which is capable of solving such a family of optimisation problems.}
\end{remark}

In Section \ref{subsec:mip_prob}, we discuss the basic notions of the \gls{mip} problems we use to apply our explainability approach: (i) the \gls{rcpsp}, a well-known problem in the optimisation literature \cite{kolisch1997psplib}; and (ii) the \gls{dpp}, a problem we recently proposed \cite{lera2025computing}.

\subsection{Optimisation Problems Domains}\label{subsec:mip_prob}

\subsubsection{The Resource-Constrained Project Scheduling Problem}\label{subsec:mip_prob:rcpsp}

Here, we report the basic definitions of the \gls{rcpsp} formalisation. Intuitively, the \gls{rcpsp} is an NP-hard problem \cite{blazewicz1983} that consists of scheduling activities to meet precedence and resource constraints while minimising the project's makespan \cite{hartmann2010,hartmann2022}. 


We base this formalisation on the single mode variant of the \gls{rcpsp} problem in \cite{kolisch1997psplib}. We consider a project with $J$ activities to complete. W.o.l.o.g. activity 1 is the only start activity, and activity $J$ is the only finish activity. Each activity $j$ has a given processing time $d_j$, and a set of $R_j$ resources that $j$ needs to complete.
Moreover, each activity $j$ has a set of activity predecessors $\mathcal{P}_j$, which implies that activity $j$ can only start when each of its predecessors $h\in \mathcal{P}_j$ has been completed. The planning horizon is divided into $T$ time units labelled as $t = 1, \dots, T$, where $T$ is a makespan upper bound. Given the precedence relations and the activity duration, we can calculate the time windows, i.e., intervals $[EF_j,LF_j]$, with the earliest and latest finishing times of each activity. 
In addition, we consider that at each time $t$, there are $K_r$ units of resource $r\in R$ available. A given activity $j$ requires an amount $k_{jr}$ of resource $r$ per period. Given those elements, the goal is to find a feasible schedule that minimises the total duration, i.e., minimises the finishing time of activity $J$.

We present the formalisation of the \gls{rcpsp} as follows.
\begin{subequations}\label{prob:RCPSP}
\begin{align}
    & \mbox{minimise} \quad \Phi(x) = \sum_{t=EF_{J}}^{LF_{J}} t\cdot x_{Jt} \label{prob:RCPSP:obj}\\
    & \mbox{s.t.} \sum_{t=EF_{j}}^{LF_{j}} x_{jt} = 1, \, j=1,\dots,J, \label{prob:RCPSP:c1}\\
    & \sum_{t=EF_{h}}^{LF_{h}} t\cdot x_{ht} \leq \sum_{t=EF_{j}}^{LF_{j}} (t-d_{j})\cdot x_{jt}, \, j=2,\dots,J,\, h \in \mathcal{P}_j, \label{prob:RCPSP:c2}\\
    & \sum_{j=1}^{J} k_{jr} \sum_{q=\text{max}(t,EF_{j})}^{\text{min}(t+d_{j}-1,LF_{j})} x_{jq} \leq K_r, \, r\in R, \, t=1,\dots,\Bar{T}, \label{prob:RCPSP:c3} \\
    & x_{jt} \in \{0,1\}, \, j=1,\dots,J, \quad t = EF_{j},\dots, LF_{j}\label{prob:RCPSP:c4}
\end{align}
\end{subequations}
where $x_{jt}$ are the decision variables expressing whether an activity $j$ is completed on time $t$. The goal of \gls{rcpsp} \eqref{prob:RCPSP:obj} is to minimise the finishing time of the last activity $J$. Constraint \eqref{prob:RCPSP:c1} requires each activity to be scheduled exactly once. Constraint \eqref{prob:RCPSP:c2} ensures the precedence relations, i.e., every activity must start after its predecessors have been completed. Constraint \eqref{prob:RCPSP:c3} limits the use of renewable resources.

\subsubsection{The Degree Planning Problem}\label{subsec:mip_prob:dpp}

For completeness, we report the basic elements of the \gls{dpp} formalisation, which we introduced in detail in \cite{lera2025computing}. 
In essence, solving the \gls{dpp} results in computing a feasible course schedule across semesters that enables a student to maximise the acquisition of skills relevant to a target job, while satisfying the typical academic regulations in Bachelor's degree syllabi.


We begin by defining a sequence of semesters $\mathcal{L} = \langle l_1,\ldots,l_n \rangle$.
Next, let $\mathcal{S} = {s_1,\ldots,s_w}$ denote the set of skills.
Finally, we refer to a \emph{unit} as a university subject, which serves as the basic component for building a degree plan. 

\begin{definition}[Unit]
A unit $u$ is represented as a tuple $\langle S, \sigma, c, \gamma \rangle$, where  
$S \subseteq \mathcal{S}$ denotes the subset of skills a student acquires by taking the unit,  
$\sigma : S \to [1 .. {lev}_{max}]$ is a mapping that assigns each skill its achieved level (${lev}_{max}$ being the highest attainable level),  
$c \in \mathbb{N}$ specifies the number of credits granted upon successful completion,  
and $\gamma : \mathcal{L} \to \{\top,\bot\}$ is a Boolean mapping indicating whether the unit is available in semester $l$ (i.e., $\gamma(l) = \top$) or not (i.e., $\gamma(l) = \bot$).\footnote{Since universities operate with limited resources, we take into consideration that not every unit is offered in each semester.}
\end{definition}
We denote by $\mathcal{U}$ the set of all units. 

Degree syllabi often include prerequisites for certain subjects, requiring students to have prior foundational knowledge before enrolling. 
Formally, we capture these precedence relationships as follows.
\begin{definition}[Precedence Relationship]
\label{def:precedence}
A precedence relationship is a tuple $p = \langle u_p, U_p \rangle$, where $U_p \subset \mathcal{U}$ is a set of units such that at least one of its units has to be completed before unit $u_p$. 
\end{definition}
We denote the set of all precedence relationships as $P$.
Additionally, we refer to \emph{core units} $\mathcal{U}^{c} \subset \mathcal{U}$ as those units that are compulsory to plan. 
Moreover, some bachelor's degree syllabi require students to complete at least one \emph{major}. A \emph{major} constitutes a specialisation within a bachelor's degree syllabus, which we now formally define:
\begin{definition}[Major]\label{def:major}
A major comprises a set of units $M \subset \mathcal{U}$ that must comply with the following restrictions: (i) a subset of core units $M^{c} \subseteq M$ must be completed; and  
(ii) a subset of eligible units $M^{e} \subset M$ is partitioned into smaller groups $SM^{e} \subseteq M^{e}$, from which a student is required to select a specified number of units $n_{SM^{e}}$ from each group. 
\end{definition}
Similarly to the set of units, $\mathcal{M}$ represent the set of majors. 
In addition, we define the concept of a \emph{credits specification} as a vector $d \in \mathbb{N}^{|\mathcal{L}|}$, indicating the number of credits students must acquire in each semester. 

Lastly, we introduce the concept of \emph{job}. Degree plans will be related to a job, since the skills gained from completing a degree plan will give access to different jobs.
\begin{definition}[Job]
A job $J$ is a tuple $\langle S,\sigma' \rangle$, where $S \subseteq \mathcal{S}$ denotes the skill set necessary for the job, and $\sigma': S \rightarrow [1 .. {lev}_{max}]$ is a function specifying the required skill level to access the job. 
\end{definition}
Having defined the concepts relevant to degree planning, we formalise the \gls{dpp} as follows.
\begin{definition}[Degree planning problem]\label{def:dpp}
Given a degree specification $\langle \mathcal{L},\mathcal{S},\mathcal{U},P,d,J \rangle$, the Degree Planning Problem is that of finding the plan for the degree that best fits the required skills' levels of the desired job $J$. Formally, if we note as $X$ the set of decision variables, then the problem is:
\begin{equation}\label{eq:degreeplan}
    x^* = \argmin_{x} distance(\sigma(S_{X}), \sigma(S_J)),
\end{equation}
where the goal is to minimise the distance between the required skills' level for the job $\sigma(S_J)$ and the skills' level acquired by completing the degree plan $\sigma(S_{U^*})$.
\end{definition}


Having introduced the basic definitions of the \gls{dpp}, we present its encoding provided by Lera-Leri et al.~\cite{lera2025computing}. As above-mentioned, the goal of the \gls{dpp} is to find the degree plan, i.e., the \emph{assignments of units to semesters}, such that it minimises the difference between the level of the required skills by the desired job and the acquired skills the student might have after completing the degree plan. Hence, given a \gls{dpp}, we formulate it as follows. 
\begin{subequations}\label{eq:dpp_mip}
    \begin{align}
        &\text{minimise} \enspace  distance(\sigma(S_{X}), \sigma(S_J)), \label{eq:obj_dpp} \\
    &\text{subject to} \enspace  \sum_{l\in\mathcal{L}}x_{u,l} \leq 1, \quad  \forall \; u \in \mathcal{U},  \label{eq:con1_dpp} \\
    & \sum_{u\in\mathcal{U}}c_u\cdot x_{u,l} = d_l, \quad \forall \; l \in \mathcal{L}, \label{eq:con2_dpp} \\
    & \sum_{l \in \mathcal{L} \; | \; \gamma_u(l) = \bot} x_{u,l} = 0, \quad\forall \; u \in \mathcal{U},  \label{eq:con3_dpp} \\
    & \sum_{l\in\mathcal{L}}x_{u,l} = 1, \quad \forall \; u \in \mathcal{U}^{c},  \label{eq:con4_dpp} \\
    & \sum_{l=l_1}^{l_{i-1}}\sum_{u\in U_p}x_{u,l} \geq x_{u_p,l_i},\quad  \forall \; p \in P, \; i=2,\dots,n, \label{eq:con5_dpp} \\
    & \sum_{M\in \mathcal{M}} v_{M} \geq 1, \label{eq:con6_dpp} \\
    & \sum_{l\in \mathcal{L}} x_{u,l} \geq v_{M},\quad  \forall \; u \in M^{c}, \; M \in \mathcal{M}, \label{eq:con7_dpp} \\
    & \sum_{u \in SM^{e}}\sum_{l\in \mathcal{L}} x_{u,l} \geq n_{SM^{e}}\cdot v_{M},  \forall \, SM^{e} \in M^{e}, \, M \in \mathcal{M}. \label{eq:con8_dpp} 
    \end{align}
\end{subequations}
We denote $x\in \{0,1\}^{|\mathcal{U}|\times |\mathcal{L}|}$ as the set of decision variables, where $x_{u,l}$ indicates whether unit $u$ is planned in semester $l$. The objective function \eqref{eq:obj_dpp} is a convex function that evaluates the distance between the required skills' level of the desired job and the skills' level acquired.\footnote{For the sake of simplicity, we provide the encoding details of the objective function in the appendix. } 
Constraints \eqref{eq:con1_dpp} ensure that each unit is not planned more than once. Then, constraints \eqref{eq:con2_dpp} ensure that units fulfil the required number of credits in each semester. In addition, constraints \eqref{eq:con3_dpp} veto plan units in semesters where they are not offered. Constraints \eqref{eq:con4_dpp} enforce the plan of core units. Furthermore, constraints \eqref{eq:con5_dpp} ensure that all precedence relationships are satisfied. 
Finally, constraints \eqref{eq:con6_dpp}, \eqref{eq:con7_dpp}, and \eqref{eq:con8_dpp} guarantee that the computed degree plan fulfils the major requirements. In detail, \eqref{eq:con6_dpp} requires completing at least one major. For that, we use auxiliary variables $v\in\{0,1\}^{|\mathcal{M}|}$, where $v_M$ indicates whether major $M$ is completed. Moreover, \eqref{eq:con7_dpp} ensures that all major core units are planned in case of pursuing major $M$. Lastly, \eqref{eq:con8_dpp} requires planning a given number of eligible units to complete major $M$. 

By formalising the \gls{dpp} as a \gls{mip}, we enable the use of state-of-the-art solvers such as CPLEX to solve the problem.
Notice that, as stated in \cite{lera2025computing}, the \gls{dpp} cannot be represented as a \gls{rcpsp}, since the optimisation goal they pursue is different. Therefore, considering these two optimisation problems provides two complementary perspectives for our evaluation.


Next, we introduce techniques for extracting the minimal set of infeasible constraints (i.e., the \gls{iis}), a fundamental step within \gls{exmip}.

\subsection{Irreducible Infeasible Subsystem}\label{subsec:iis}

The \acrfull{iis} is a minimal set of inconsistent constraints. 
However, to properly define the notion of \gls{iis}, we must first define the notions of constraint and system of constraints. We borrow the definition of constraint from \cite{meseguer2006soft}:
\begin{definition}
    Given a sequence of distinct variables $V=\langle x_1,\dots,$ $x_k\rangle$ and their associated finite domains $\mathcal{D}_1,\dots,\mathcal{D}_k$, a constraint $c$ is a subset of $\mathcal{D}_1\times\dots\times\mathcal{D}_k$. Thus, a constraint $c$ specifies the assignments allowed by the variables of $V$.
\end{definition}
In other words, a constraint is a condition that variables must satisfy. 
Then, a constraint system is a set of constraints $C$. 
Hence, we define an infeasible system as follows.
\begin{definition}
    An Infeasible System of constraints is a set of constraints $C$ such that no assignment of $V$ exists that satisfies all the constraints in $C$.
\end{definition}
Having defined an infeasible system of constraints, we define an \gls{iis} borrowing the definition from \cite{chinneck2007feasibility}:

\begin{definition}\label{def:iis}
    An \gls{iis} is a subset of constraints within an infeasible system that is in itself infeasible, but any proper subset of the \gls{iis} is feasible. 
\end{definition}

Finding an \gls{iis} is crucial in \gls{csp} 
because it identifies the constraints preventing us from obtaining a feasible solution. Thus, many algorithms have been developed to find the \gls{iis} for \gls{lp} \cite{van1981irreducibly,gleeson1990identifying,parker1996finding,chinneck1997finding}, \gls{milp} \cite{guieu1999analyzing} or Non-Linear Integer Programming \cite{chinneck1995analyzing}. Finding the \gls{iis} involves iteratively analysing the infeasible system and identifying the minimal subset of constraints that cause the infeasibility. Some examples of algorithms that find \glspl{iis} are: the Deletion Filter \cite{chinneck1991locating}, which iteratively deletes constraints and checks the feasibility of the system; the Additive Method \cite{tamiz1996detecting}, which iteratively adds constraints to an initially empty set and tests whether the set is feasible; or the combination of both methods \cite{guieu1999analyzing}. 
Although computing \glspl{iis} can be costly, state-of-the-art solvers efficiently compute them.

\section{Building Explanations for Optimisation-Based Decisions}\label{sec:expl}

This section presents \gls{exmip}, our methodology for generating contrastive explanations for users interacting with optimisation-based \gls{dss}. As illustrated in Figure~\ref{fig:overview}, our approach combines constraint reasoning techniques with \glspl{llm} to produce textual explanations. Specifically, the figure outlines the key components of \gls{exmip}, which we explain in detail in the remainder of this section. While the methodology is general ---applicable to any application domain involving a \gls{mip} problem--- we use the \gls{dpp}, formalised in Section \ref{subsec:mip_prob:dpp}, as an illustrative example throughout this section.

In a \gls{mip} problem, constraints encode restrictions that the solution must satisfy, while the objective function encodes some preferences or quantities to be optimised \cite{boyd2004convex}. Thus, a solver output is affected by the constraints and the objective function. 
In the case of the \gls{dpp}, the solver output is the degree plan, which is affected by the syllabi restrictions and the skills required for the desired job. 
Hence, we propose to exploit the information provided by the constraints and the objective function to build contrastive explanations. 
The following running example illustrates the concepts of our approach.

\begin{example}\label{example:knapsack}
    Consider that we aim to compute a degree plan for a real-world Bachelor's degree in Information and Communication Technology for a student pursuing a career in Software Engineering.\footnote{Section \ref{subsec:prob_instances} provides details on the problem data.} Such a Bachelor's Degree lasts six semesters (the first semester starts in autumn), and students must complete 40 credits each semester (completing a subject earns 10 credits). 
    As mentioned in Section \ref{subsec:mip_prob:dpp}, solving the \gls{dpp} amounts to maximising the required skills' level of the desired job, 
    while satisfying a set of constraints $C$. 
    For the sake of clarity, we only show a part of the optimal degree plan $\textbf{x}^{*}$ in Table \ref{tab:solution_dpp}.
\end{example}

\begin{table*}[t]
    \caption{Information of the key subjects of the optimal degree plan of Example \ref{example:knapsack}. We provide the semesters in which the subject is planned (and the season), the subject's seasonal offer, its predecessors, and whether the subject is a core subject. If a subject is not planned, we indicate it with a $-$ sign in the semester column. }
    \label{tab:solution_dpp}
    \setlength{\tabcolsep}{5pt}
    \centering
    \begin{tabularx}{0.9\textwidth}{p{0.5\textwidth}cccc}
        \toprule
        Subject & Semester & Offer & Predecessors & Core \\
        \midrule
        Programming Fundamentals (PF) & 2 (Spring) & Both & - & Yes \\
        Programming Techniques (PT) & 4 (Spring) & Spring & PF & Yes \\
        Object Oriented Programming (OOP) & $-$ & Spring & PF & No \\
        Operating Systems Programming  (OSP) & 5 (Autumn) & Autumn & PT or OOP & Yes \\
        \bottomrule
    \end{tabularx}
\end{table*}
\normalsize  

\begin{table*}[t]
    \setlength{\tabcolsep}{3pt}
    \caption{Set of possible user questions, query constraints and constraint encodings for the \gls{dpp} formalisation.}
    \label{tab:queries_dpp}
    
    \begin{tabularx}{\textwidth}{p{0.025\textwidth}p{0.37\textwidth}p{0.33\textwidth}Y}
        \toprule
        & \textbf{Question} & \textbf{Query constraints} & \textbf{Constraint encoding ($C_Q$)} \\
        \midrule
        Q1.& Why is subject $u$ planned at semester $l$? & \textbf{veto} to plan $u$ at semester $l$ & ${x}_{u,l} = 0$ \\
        Q2.& Why is subject $u$ not planned at semester $l$? & \textbf{enforce} to plan $u$ at semester $l$ & ${x}_{u,l} = 1$ \\
        Q3.& Why is subject $u$ \textbf{not} planned before semester $l$? & \textbf{enforce} to plan $u$ before semester $l$ & $\sum_{i=0}^{l-1} {x}_{u,i} = 1$ \\
        Q4.& Why is subject $u$ \textbf{not} planned after semester $l$? & \textbf{enforce} to plan $u$ after $l$ & $\sum_{i=l+1}^{l_n}{x}_{u,i} = 1$ \\
        Q5.& Why is subject $u$ planned? & \textbf{veto} $u$ & $ \sum_{l\in\mathcal{L}}{x}_{u,l} = 0$ \\
        Q6.& Why is subject $u$ not planned? & \textbf{enforce} $u$ & $\sum_{l\in\mathcal{L}} {x}_{u,l} = 1$ \\
        Q7.& Why is the group of subjects $U$ planned at semester $l$? & \textbf{veto} to plan all $u\in U$ at semester $l$ & $\sum_{u\in U}{x}_{u,l} < |U|$ \\
        Q8.& Why is the group of subjects $U$ \textbf{not} planned at $l$? & \textbf{enforce} to plan all $u\in U$ at semester $l$ & $\sum_{u\in U}{x}_{u,l} = |U|$ \\
        Q9.& Why is subject $u$ planned at semester $l$ instead of $u'$? & \textbf{enforce} to plan $u'$ at $l$ and \textbf{veto} $u$ at $l$ & ${x}_{u',l} = 1$, ${x}_{u,l} = 0$ \\
        Q10.& Why is subject $u$ planned instead of subject $u'$? & \textbf{enforce} $u'$ and \textbf{veto} $u$ & $\sum_{l\in\mathcal{L}}{x}_{u',l} = 1$, $\sum_{l\in\mathcal{L}}{x}_{u,l} = 0$ \\
        \bottomrule
    \end{tabularx}
\end{table*}
\normalsize

Some approaches in the literature build a contrastive explanation by comparing an optimal solution with a user-desired solution \cite{georgara2022building,zehtabi2024contrastive}. The latter results from translating a user's ``\emph{query}'', i.e., a question for which a user expects an explanation, into \emph{additional constraints} that the decision (optimal solution) should satisfy. 
While our methodology does not directly compare the two solutions as in previous work, we similarly encode a user's query as additional constraints on the optimisation problem--- a process we describe in detail in the next section.

\subsection{Query Translation}\label{sec:query}

Here, we characterise the types of queries that a user can pose and how to encode them into query constraints. 
More precisely, a user asks ``\emph{why}'' questions (in natural language) that challenge a decision or solution computed by an optimisation-based \gls{dss}. 
For example, regarding the \gls{dpp}, a user might ask ``Why is subject Programming Fundamentals not planned at semester $4$?''. 
Hence, a contrastive explanation should tell the user why an alternative solution complying with the query, the so-called user-desired solution, is not at least as good as the optimal solution. In the previous example, this alternative solution would be a degree plan in which Programming Fundamentals is scheduled for semester 4, as suggested by the user's query.
Henceforth, since a user's query challenges a recommendation computed by an optimisation-based \gls{dss}, we assume that in the case of a \gls{mip} problem, a user's query refers to value assignments that do occur in an optimal solution \cite{zevcevic2021causal}. 
In other words, a user's query is related to the decision variables of the \gls{mip} at hand.
In optimisation problems, decision variables are related to objects, activities, agents, etc. 
For instance, in the \gls{dpp}, the decision variables indicate whether a subject is planned for a given semester. Hence, there is a decision variable for every pair of subjects and semesters. 

\begin{example}\label{example:KP2}
    Following Example \ref{example:knapsack}, a user may ask ``\emph{Why is \emph{Programming Fundamentals} not planned at semester 4?}'' about the solution in Table \ref{tab:solution_dpp}. Building an explanation to this query requires enforcing to plan Programming Fundamentals at semester 4, which is encoded by constraint   \begin{equation}\label{eq:rcpsp:query:example}
        x_{\text{PF},4} = 1,
    \end{equation} 
    (see Table \ref{tab:queries_dpp}). 
    Enforcing this constraint would allow obtaining a user-desired solution. 
\end{example}

To formally encode a user’s query as additional constraints in the \gls{mip} problem, we adopt the classification proposed by Georgara et al.~\cite{georgara2022building}.
\begin{itemize}
\item \textbf{enforce} constraints to express that some value assignments that do not hold in an optimal solution are required to hold in the user-desired solution; and 
\item \textbf{veto} constraints to express that some value assignments that do hold in an optimal solution must not hold in the user-desired solution.
\end{itemize}

Following such classification, Table \ref{tab:queries_dpp} lists possible query templates for the \gls{dpp}. The table also shows the translations of queries into (linear) query constraints. 
While the table focuses on the \gls{dpp}, the process of (i) identifying queries, (ii) defining query templates, and (iii) translating them into query constraints is general and can be applied to other \gls{mip} problems. Indeed, we report the lists of possible query templates for the \gls{rcpsp} in Table \ref{tab:queries_rcpsp} of Appendix \ref{subsec:query_translation}. 

To illustrate how user queries are systematically translated into formal query constraints within our approach, we now discuss concrete examples based on the \gls{dpp} (see Table~\ref{tab:queries_dpp}). 
\begin{itemize}
    \item Query Q1 ``\emph{Why is subject $u$ planned at semester $l$?}'' translates into \textbf{vetoing} the planning of subject $u$ at semester $l$ (i.e., $x_{u,l}=0$) to enforce an alternative solution to the optimal solution. Notice that this query refers to a \textbf{single assignment} occurring in the optimal solution. 
    \item To comply with query Q7 ``\emph{Why is the group of subjects $U$ planned at semester $l$?}'' we must \textbf{veto} the planning of the entire set of subjects $U$ at semester $l$ (i.e., $\sum_{u \in U} x_{u,l} < |U|$). In this case, the query refers to the planning of multiple subjects, i.e., \textbf{multiple assignments}, in the optimal solution. 
    \item When addressing query Q2 ``\emph{Why is subject $u$ \textbf{not} planned to be completed at semester $l$?}'' we must \textbf{enforce} the hypothetical planning of unit $u$, which the current (optimal) degree plan disregards. 
    \item To answer the query Q8 ``\emph{Why is subject $u$ planned to be completed at semester $l$ instead of subject $u'$?}'', we must: (1) enforce subject $u'$; and (2) veto $u$ at semester $u$. Thus, we consider combinations of \textbf{enforce} and \textbf{veto} constraints to answer more specific queries about value assignments that did and did not occur. 
\end{itemize}

\subsection{The User-Desired Satisfiability Problem}\label{subsec:asp}

Since numerical comparisons are not desirable in explanations \cite{miller2019explanation}, we propose instead to identify the reasons why a user query, which challenges the optimal solution by posing a user-desired scenario, cannot yield a better solution than the optimal one. 
To do so, we cast the problem of obtaining the user-desired solution as a \gls{csp}. Such a \gls{csp} contains the problem constraints, the query constraints, and an additional constraint enforcing a solution at least as good as the optimal solution to the ``\emph{main}'' problem. 
This formulation will serve to extract the constraints that restrict us from obtaining a user-desired solution with optimal value. 

\begin{definition}[User-Desired Satisfiability Problem]\label{def:asp}
    Let $M$ be a \gls{mip}, a set of query constraints $C_Q$ associated with a query, and $f^{*}=f(\textbf{x}^{*})$ the value of the optimal solution $\textbf{x}^{*}$ for $M$. The \gls{asp} for $M$ is a \gls{csp} composed of the following constraints: 
\begin{equation}\label{eq:asp}
    (i)\; f(\textbf{x}) \leq f^{*}, \quad (ii)\; C, \quad (iii)\; C_Q,
\end{equation}
where $\textbf{x}$ are variables and (i) is a minimality constraint.\footnote{We consider $M$ a minimisation problem. In maximisation problems, we impose a maximality constraint such that $f_M(\textbf{x})  \geq f^{*}$.}
\end{definition}
\noindent

\begin{example}\label{example:ASP}
We define a \gls{asp} for example \ref{example:KP2} as the following \gls{csp} with variables $\textbf{x}$:    
\begin{equation}\label{eq:asp:example}
    (i)\; f(\textbf{x}) \leq f^{*}, \quad (ii)\; C, \quad (iii)\; x_{\text{PF},4} = 1,
\end{equation}
    where $(i)$ is the minimality constraint, $(ii)$ are the main problem constraints, and $(iii)$ is the query constraint.
\end{example}

Based on a user's query, the literature typically considers user-desired optimisation scenarios that accommodate such a query, 
aiming at quantifying how worse a user-desired solution would be in terms of optimality. 
In this context, the literature typically defines an extended optimisation problem, adding the query constraints to the constraint set \cite{zehtabi2024contrastive,georgara2022building}. Then it solves the extended problem and compares its solution with the optimal one. 
Notice that, in this paper, we neither solve an extended \gls{mip} problem that accommodates a user's query (as in \cite{zehtabi2024contrastive}) nor quantify how worse the user-desired solution is 
(as in \cite{georgara2022building}). 
In contrast, we identify the reasons why the user's query leads to worse solutions. In the next section, we discuss how to compute such reasons from the constraints of the \gls{asp}.

\subsection{Computing the Infeasible Constraints}\label{subsec:comp_iis}

Having defined the \gls{asp}, we extract the minimal set of constraints that prevent us from finding a feasible solution to the \gls{asp}, i.e., we extract the \gls{iis}.
Notice that there might be multiple \glspl{iis} corresponding to the same set of constraints since the number of \glspl{iis} increases exponentially with the size of a constraint set \cite{chakravarti1994some}. 
More precisely, different \glspl{iis} (despite being all ``irreducible'') could contain different numbers of constraints (i.e., some \glspl{iis} might be smaller than others). 
Hence, this would subsequently lead us to explanations of varying sizes, since we build our explanations based on the constraints of an \gls{iis}.  
Following \cite{rosenfeld2021better},\footnote{The number of ``rules'' (i.e., constraints in our case) in an explanation is a commonly used metric to evaluate its simplicity.} 
we should aim to compute the smallest possible \gls{iis} because the more concise an explanation, the easier it is to understand.

On the one hand, some algorithms in the literature can compute the ``smallest'' \gls{iis} \cite{liffiton2009branch,ignatiev2015smallest}. However, computing the ``smallest'' \gls{iis} is very costly.
On the other hand, off-the-shelf solvers such as CPLEX use specialised heuristics to efficiently compute a single \gls{iis} that is not guaranteed to be the smallest one. 
However, in practice, it is observed that CPLEX computes the smallest \gls{iis} for most instances with a significantly lower computational cost \cite{leraleri2025exploiting}.

Having computed the IIS, the resulting IIS might give us information about the \gls{asp} instance at hand. If an \gls{asp} is \textbf{feasible}, i.e., at least one solution exists, the alternative proposed by the user is as good as the optimal solution of the main problem. In such a case, there exists no set of inconsistent constraints, i.e., $IIS(C_{\gls{asp}}) = \emptyset$, where $C_{\gls{asp}}$ is the set of constraints of an \gls{asp}.\footnote{We denote $IIS(C)$ the function computing an \gls{iis} of constraint set $C$.} 
If the \gls{asp} is infeasible, i.e., no solution exists, we differentiate two scenarios: 

\begin{itemize}
     \item \textbf{Query Infeasibility}. In this case, the query constraints are incompatible with the main problem restrictions. As a result, some query constraints conflict with the problem constraints in \eqref{eq:asp}, rendering the problem infeasible. Consequently, the \gls{iis} contains a set of query constraints and main problem constraints, i.e., 
     $IIS(C_{\gls{asp}}) = {C'}_{Q} \cup {C'}$ with ${C'}_{Q}\subseteq C_{Q}$, ${C'}\subseteq C$, and ${C'}_{Q}, {C'} \neq \emptyset$. 

     \item \textbf{Minimality Infeasibility}. Here, some query constraints conflict with some main problem constraints, and the minimality constraint in \eqref{eq:asp}. More intuitively, while the user's query is compatible with the problem's restrictions, it yields a less optimal solution. 
     Consequently, the \gls{iis} contains a set of query constraints, main problem constraints, and the minimality constraint. 
\end{itemize}
 
\begin{table}[t]
    \caption{$IIS$ from Example \ref{example:ASP}. We indicate the constraint types (in different colours) and the elements of each constraint.}
    \label{tab:iis:example}
    \centering

    \begin{tabularx}{\columnwidth}{ccl}
    \toprule
     Id &  Type &  Values \\
    \midrule
     $q$ &  Query &  Programming Fundamentals, Semester {$4$} \\
     $c_1$ &  \textcolor{olive}{Completion} &  Programming Fundamentals  \\
     $c_2$ &  \textcolor{orange}{Offer} &  {Programming Techniques} \textbf{not offered} in Autumn  \\
     $c_3$ &  \textcolor{orange}{Offer} &  {Object Oriented Prog.} \textbf{not offered} in Autumn  \\
     $c_4$ &  \textcolor{orange}{Offer} &  {Operating System Prog.} \textbf{not offered} in Spring  \\
     $c_5$ &  \textcolor{blue}{Precedence} &  {Prog. Fund.} \textbf{before} {Prog. Techniques} \\
     $c_6$ &  \textcolor{blue}{Precedence} &  {Prog. Fund.} \textbf{before} {Object Oriented Prog.} \\
     $c_7$ &  \textcolor{blue}{Precedence} &  {Prog. Tech.} \textbf{or} {Obj. Or. Prog.} \textbf{before} {OS Prog.} \\
     $c_8$ &  \textcolor{red}{Core} &  {Operating System Programming} \\
    \bottomrule
    \end{tabularx}
\end{table}
\normalsize

\begin{example}\label{example:KP3}
    Following example \ref{example:ASP}, Table \ref{tab:iis:example} shows the $IIS$ for the \gls{asp} in equation \eqref{eq:asp:example}. In this example, the $IIS$ contains 8 main problem constraints (1 completion constraint, 3 offer constraints, 3 precedence constraints, 1 core constraint) and the query constraint. Hence, since the minimality constraint is not within the $IIS$, this case belongs to the Query Infeasibility category. For completeness, Table \ref{tab:iis:example:math} in Appendix \ref{iis:example} shows the mathematical encoding of the constraints within the \gls{iis}. 
\end{example}

\subsection{Constructing the \emph{Graph of Reasons}}\label{subsec:ordering}

\begin{figure*}[t]
        \centering
    \includegraphics[width=\textwidth]{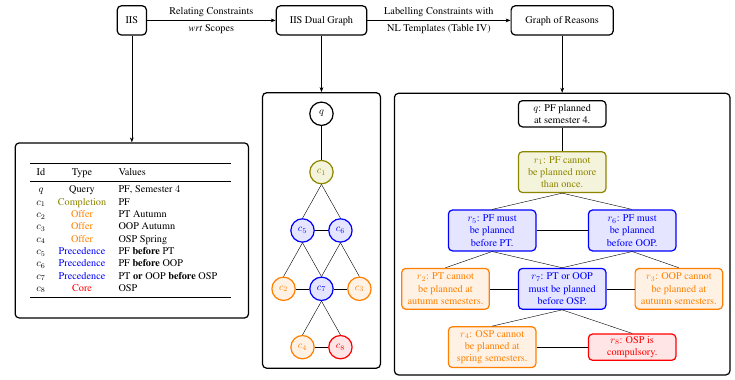}
    \caption{Process to build a graph of reasons applied to our running example: (1) relating constraints within the IIS to yield its dual graph, and (2) labelling constraints in the dual graph to build the graph of reasons.}
\label{fig:megafig}
\end{figure*}

\begin{table}[b]
    \caption{Natural Language Templates for each type of \gls{dpp} constraint in \eqref{eq:dpp_mip}. We report the Natural Language Templates for each type of \gls{rcpsp} constraint in Table \ref{tab:templates_rcpsp} in Appendix \ref{subsec:rcpsp_constraint_labels}. }
    \centering
    \label{tab:templates_dpp}
    \setlength{\tabcolsep}{3pt}
    
    \begin{tabularx}{\columnwidth}{lp{0.66\columnwidth}}
        \toprule
        (i) \textcolor{olive}{Completion} \eqref{eq:con1_dpp}: & ``\emph{Subject $u$ cannot be planned more than once}'' \\
        (ii) \textcolor{cyan}{Credits} \eqref{eq:con2_dpp}: & ``\emph{$d_l$ credits must be planned at semester $l$}'' \\
        (iii) \textcolor{orange}{Offer} \eqref{eq:con3_dpp}: & ``\emph{Subject $u$ cannot be planned at autumn/spring semesters}'' \\
        (iv) \textcolor{red}{Core} \eqref{eq:con4_dpp}: & ``\emph{Subject $u$ is compulsory}'' \\
        (v) \textcolor{blue}{Precedence} \eqref{eq:con5_dpp}: & ``\emph{Subject $h$ must be planned before Subject $j$}'' \\
        (vi) \textcolor{brown}{Major} \eqref{eq:con6_dpp}: & ``\emph{Students must complete at least one major}'' \\
        (vii) \textcolor{yellow}{Major core} \eqref{eq:con7_dpp}: & ``\emph{Subject $u$ is compulsory to complete major $M$}'' \\
        (viii) \textcolor{magenta}{Eligible} \eqref{eq:con8_dpp}: & ``\emph{$n_{SM^{elig}}$ eligible subjects from the set of subjects {U} must be complete to obtain major {M}}'' \\
        \bottomrule
    \end{tabularx}
\end{table}
\normalsize

Having established how we compute the set of infeasible irreducible constraints (i.e., the \gls{iis}), we now address the challenge of relating and providing a semantic interpretation to the constraints within an \gls{iis} in a way that enables the \gls{llm} to generate more precise and meaningful explanations ---ones that convey why the user query cannot be satisfied.
This is necessary because, as we later demonstrate in our experiments, inferring the relationships between constraints from an \gls{iis} to build explanations is not an easy \emph{reasoning} task for an \gls{llm}.
Indeed, \glspl{llm} might hallucinate and generate factually incorrect explanations when tackling such a complex task because they have limitations with abstract reasoning \cite{gendron2023large} and summarisation \cite{afzal2023challenges}.

Along these lines, we show that, by exploiting the structural relations among constraints, we can organise them in a graph, the so-called \emph{graph of reasons}, whose structure eases the understanding of the conflicts between constraints within the \gls{iis} (in contrast with the unstructured list of constraints in the \gls{iis}). Ultimately, our purpose is to use such a graph as a prompt to make the knowledge inherent in the constraints available to the \gls{llm}. This enables the \gls{llm} to generate explanations that are more complete and readable. 

To this end, we rely on well-known concepts from the Constraint Reasoning literature, as explained hereafter. 
Hypergraphs \cite{dechter2022reasoning} represent how a set of constraints relates to decision variables. 
A \emph{hypergraph} 
$\mathcal{H}=(V,S)$ is a pair where $V$ is a set of vertices, and $S=\{S_1,\dots,S_l\}$, $S_i\subseteq V$, is a set of subsets of $V$ called hyperedges. Notice that the vertices represent the set of decision variables of our model, and a hyperedge $S_i$ is the scope of constraint $c_i$, which is the subset of decision variables involved in the constraint.   
However, hypergraphs do not directly represent the relations among constraints; instead, this is the purpose of dual graphs.
Formally, a \emph{dual graph} \cite{dechter2022reasoning} $D=\{S,I\}$ is an undirected graph whose vertices $S$ are the hyperedges of $\mathcal{H}$ and $(S_i,S_j)\in I$ iff $S_i \cap S_j \neq \emptyset$. 
Intuitively, each node of a dual graph corresponds to a constraint, and each edge indicates that two constraints share some variables. 
Therefore, the dual graph of an \gls{iis} captures the relations between its constraints. 


Figure \ref{fig:megafig} showcases our approach for building the graph of reasons from the \gls{iis} in Table \ref{tab:iis:example}.
For instance, constraint $c_1$, which ensures that Programming Fundamentals is planned at most once, is related to the query constraint enforcing the planning of Programming Fundamentals in the fourth semester because they share decision variables; hence, there is an edge connecting them. More formally, both constraints are related because they share decision variables; hence, there is an edge connecting them in the graph of reasons.
In contrast, other constraints such as $c_2$, which veto the plan Programming Techniques in Autumn Semesters, are not \textit{directly} related to the query constraint. Nonetheless, $c_2$ is linked to the precedence constraint $c_5$, which ensures that Programming Fundamentals is planned before Programming Techniques. 
Since $c_5$ is also related to $c_1$, an indirect relationship between $c_2$ and the query constraint can be established through this chain of dependencies, formally represented by a path of edges between $c_2$ and the query constraint.

After having established the dual graph of constraints, following the approach by \cite{pozanco2022explaining}, we use \gls{nl} templates to provide a semantic interpretation to each constraint in such a graph, as shown in Figure \ref{fig:megafig}.\footnote{Notice that translating queries and labelling constraints using \gls{nl} templates are the only problem-dependent steps that must be carried out by an expert, in contrast with the rest of our approach, which is domain-independent.
As an example, Table \ref{tab:templates_dpp} shows the \gls{nl} templates involved in our running example for each type of \gls{dpp} constraints in Section \ref{subsec:mip_prob:dpp}. We show the \gls{nl} templates for each type of \gls{rcpsp} constraints in the appendix. } 
Ultimately, a ``\emph{reason}'' takes the form of the semantic interpretation of a constraint by means of a \gls{nl} template. 
Therefore, we refer to the dual graph of the semantically-characterised constraints as the \emph{graph of reasons}. 

A desirable property that characterises the graph of reasons is that it is not ``disconnected'', i.e., it involves reasons that are always related to each other. This property can be directly expressed in terms of the graph-theory concept of \emph{connectedness}, as formalised by Theorem \ref{theo:con}.

\begin{theorem}\label{theo:con} 
    Given an \gls{iis}, the \gls{iis} dual graph $D$ is connected.
\end{theorem}

\begin{proof}
    We prove by contradiction that a disconnected \gls{iis} dual graph, i.e., a graph with multiple disconnected components, is not possible to obtain due to the infeasible and irreducible properties of an \gls{iis} (Definition \ref{def:iis}). 
    
    Consider an \gls{iis} graph that is disconnected in two subgraphs, $A = (C_A,E_A)$ and $B = (C_B,E_B)$ such that
    \begin{align*}
        C_A \cup C_B &= IIS, \\
        C_A \cap C_B &= \emptyset, \\
        S_A \cap S_B &= \emptyset,
    \end{align*}
    where $S_A,S_B$ are the scopes of the set of constraints $C_A$ and $C_B$ respectively. Notice that, since the graph is disconnected, $A$ and $B$ subgraphs do not share variables, i.e., $S_A \cap S_B = \emptyset$. 
    Then, since the subset of constraints $C_A$ and $C_B$ is (disjoint) subsets of the \gls{iis}, they are feasible (see Definition \ref{def:iis}). 
    However, since the two constraint sets have different variables (domains), their union would not affect the feasibility of the resulting system of constraints. 
    As a result, we would have a feasible \gls{iis}, 
    which is a contradiction in itself.
    Therefore, it is not possible to have a disconnected \gls{iis} graph. 
\end{proof}

\subsection{Providing Textual Explanations with LLMs}
\label{subsec:text}

\glspl{llm} have emerged as a powerful tool for a wide range of tasks, including, for example, machine translation, reasoning, summarisation, and question-answering \cite{hadi2023survey,zhao2023survey}. 
However, it is well-known that \glspl{llm} present limitations with abstract reasoning \cite{gendron2023large} and summarisation \cite{afzal2023challenges}.

As far as we are concerned, a consequence of these limitations is that \glspl{llm}' responses tend to take the implicit assumption made in the user prompt for granted, rather than independently verifying their truthfulness. 

To illustrate this phenomenon, we revisit the degree plan from Example \ref{example:knapsack} and consider the following user query: ``\textit{Why is Programming Fundamentals not planned at semester 3?}''. 
Notice that, in contrast with the query in Example \ref{example:KP2} ``\textit{Why is Programming Fundamentals not planned at semester 4?}'' ---suggesting an invalid alternative solution (i.e., one that violates some hard constraints)---, this query suggests an alternative solution that is equally valid and optimal, i.e., a solution that does not violate any constraint and that has the same optimal total utility. 
In this case, our approach yields an empty \gls{iis}, correctly representing the fact that the user-suggested alternative is equally good.\footnote{In such cases, the explanation provided by \gls{exmip} consists of indicating that the alternative suggested yields a degree plan that is as good as the optimal one. Additionally, we can provide a solution compliant with the user query by solving the resulting \gls{asp} discussed in Section \ref{subsec:asp}.} 
In contrast, when straightforwardly prompting an \gls{llm} to generate an explanation in response to the query at hand, the \gls{llm} produces the following explanation:\footnote{For this test, we provided the \gls{llm} only with the problem solution and the instance data. We show the first sentence of the explanation for the sake of brevity. } ``\textit{Programming Fundamentals is not planned in semester 3 because doing so would violate prerequisite timing for other subjects that are scheduled earlier in the plan...}'' 
Notice that the \gls{llm} implicitly assumes that the suggested alternative is indeed invalid, without independently verifying the truthfulness of this assumption. In fact, \gls{llm} are notoriously reluctant to contradict the user, a phenomenon ---known as ``\textit{sycophancy}''~\cite{sharma2023towards,perez2023discovering}--- that can lead to false statements, as in this case.
Notice that this aspect is particularly important in the context of explaining the solutions of optimisation problems, for which multiple, equally good solutions may exist, as discussed in Section~\ref{subsec:comp_iis}.

Against this background, our methodology, which provides the information encoded by the \gls{iis} as a structured, semantically-characterised graph, as discussed in the above section, is key to guiding \glspl{llm} in generating \emph{correct} textual explanations.

In practice, we build a prompt that uses the graph of reasons as an input to the \gls{llm} to generate such an explanation. We employ a simple zero-shot prompt strategy \cite{sahoo2024systematic} that contains the following information.

\renewcommand{\theenumi}{\roman{enumi}} 
\renewcommand{\labelenumi}{\theenumi.}

\begin{enumerate}
    \item A \textbf{context} of the optimisation problem to solve. In our case, a description of the \gls{dpp}, its goal, and its restrictions. Additionally, we provide the optimal solution, i.e., the optimal degree plan. 
    \item The \textbf{question} posed by the user. 
    \item A \textbf{textual encoding} of the graph of reasons. Since the graph of reasons is a visual representation, we include a textual encoding of it in the prompt. This encoding specifies the node labels ---representing the reasons--- and the relationships between them, corresponding to the graph's edges. 
    \item The \textbf{instructions} of the task to do, i.e., writing a structured explanation using the encoding of the graph of reasons. 
\end{enumerate}

A myriad of prompt engineering methods have been discussed in the literature. 
For instance, techniques such as chain-of-thought prompting \cite{wei2022chain} and prompt ensembling \cite{zhao2021calibrate} have demonstrated noteworthy improvements in tasks requiring multi-step reasoning or factual accuracy. 
Contrarily, for other, more general tasks such as translation, zero-shot prompts can significantly outperform few-shot prompts due to semantic contamination from the examples (shots) \cite{reynolds2021prompt}. 
Additionally, minor variations in phrasing or formatting can lead to substantial differences in model outputs across a wide range of tasks \cite{wei2022chain, zhao2021calibrate}. 
This variability highlights the importance of prompt engineering.

Despite the variety of prompt engineering methods discussed in the literature, we adopt a zero-shot prompting strategy, as the scope of this paper is to showcase how \glspl{llm} can be leveraged to provide explanations, rather than to conduct an extensive evaluation of different prompting strategies.

\begin{example}\label{example:text:dpp}
    Considering the computed \gls{iis} and its graph of reasons in Figure \ref{fig:megafig}, we provide the textual explanation generated with a \gls{llm} to answer the question ``\emph{Why is Programming Fundamentals not planned at semester 4?}''. The resulting explanation is the following.\footnote{We chose \emph{gpt-4o} to compute textual explanations because it was the top LLM in the Arena Leaderboard (\url{https://arena.ai/leaderboard}) at the time of running the experiments (March 2025). } 
    ``\emph{Programming Fundamentals is not planned in semester 4 because it must be taken earlier to satisfy the prerequisite requirements of later subjects, particularly Programming Techniques and Object Oriented Programming. These subjects, in turn, need to be completed before taking Operating Systems Programming, a compulsory subject that can only be taken in the autumn. Additionally, Programming Techniques and Object-Oriented Programming cannot be scheduled in the autumn semesters, further constraining their placement. Since subjects cannot be scheduled more than once, Programming Fundamentals must be scheduled early enough to enable this sequence, which would make semester 4 too late for it.}''
    
\end{example}
Notice that the explanation in Example \ref{example:text:dpp} is correct: there are no imprecisions or mistakes, and it is complete, mentioning all the reasons in the graph. 
In appendix \ref{sec:prompt}, we discuss the exact prompt used in this example.

\section{Experimental Evaluation}
\label{sec:exp}

The goal of our experimental evaluation is twofold: we aim to evaluate (i) whether our approach can produce explanations that are complete and easy to understand (Section \ref{subsec:quality}), and (ii) the empirical hardness of computing our explanations, allowing us to characterise which problem instances can be tackled in real-time for users (Section \ref{subsec:overhead}).

\subsection{Experimental Setting}\label{subsec:data}

We now describe our experimental setting, including the considered optimisation\allowbreak-based benchmark (Section \ref{subsec:prob_instances}) scenarios and the methodology used to conduct our experiments (Section \ref{subsec:exp_methodology}). 

\subsubsection{Optimisation-Based Benchmark Scenarios}\label{subsec:prob_instances}

We test our approach in two different optimisation problems, the \gls{dpp} and the \gls{rcpsp}. We do so to (i) showcase its generality and applicability to real-world problems and (ii) evaluate its scalability with computationally hard problem instances. In this section, we detail the problem instances used in our experimental evaluation. 

\paragraph{\gls{dpp}}
In our first benchmark problem, we consider a real-world dataset previously employed in \cite{lera2025computing}. The dataset comprises a bachelor's degree in Information and Communication Technology (ICT) from Western Sydney University, including $|\mathcal{U}| = 82$ units, of which $|\mathcal{U}^{c}| = 16$ are core units, considering 7 major programs. Each unit provides $10$ credits. Moreover, based on the \gls{sfia} skill framework \cite{sfia_reference}, units have been mapped with the acquired level upon completion of professional skills. Furthermore, a student must complete $c_T = 240$ credits to obtain the degree, typically over 6 semesters with 40 credits per semester.
Lastly, we consider a real-world dataset of $118$ jobs from the \emph{APS Digital Career Pathways} dataset \cite{job_dataset}, which specifies the required \gls{sfia} skill levels for each position.

\paragraph{\gls{rcpsp}}
Among the many variants of the \gls{rcpsp} \cite{hartmann2022}, our second benchmark problem is the single-mode \gls{rcpsp}, in which each activity can be completed in only one way. 
The problem instances used come from the \gls{psplib} \cite{kolisch1997psplib}.\footnote{\url{https://www.om-db.wi.tum.de/psplib/}. Last consulted: April 2024.} Such a library employs ProGen \cite{kolisch1995characterization}, an algorithm that generates problem instances with varying parameters, such as the number of activities, network complexity, or resource strength. We use 1440 instances with $|J|\in \{30,60,90\}$, generated with 48 combinations of parameters (see \cite{kolisch1997psplib}) and 10 different seeds.

\subsubsection{Experimental Methodology}\label{subsec:exp_methodology}

As our first step, we solve the benchmark problem instances obtained as described in Section \ref{subsec:prob_instances}.
We then generate queries as described in Section \ref{sec:query} (10 and 8 query types for the \gls{dpp} and \gls{rcpsp}, respectively). 
As for the \gls{dpp}, we generate 5 queries for each query type, for a total of 50 queries per optimal solution. Since we use a dataset of $118$ jobs (each with a different optimal solution), we generate $5900$ queries in total.
Once the queries are generated, we generate their corresponding \glspl{asp} according to Section \ref{subsec:asp}, and we calculate their \glspl{iis} according to Section \ref{subsec:comp_iis}.
Finally, we compute the graphs of reasons according to Section \ref{subsec:ordering}.

We solve all instances and compute \glspl{iis} with CPLEX v22.1.0.\footnote{Code and data for the experimental evaluation available at: \url{https://github.com/RogerXLera/LLMxMIP}.} 
We run our tests on a machine with a 10-core, 2.2 GHz CPU and 16GB of RAM. We limit solving time to 10 hours per problem instance and per explanation. 

Then, we proceed to generate our \gls{llm}-generated textual explanations that answer the above-mentioned queries.
Since we aim to evaluate whether the structure conveyed by the graph of reasons is beneficial to the \gls{llm},\footnote{We use \emph{gpt-4o} (\url{https://openai.com/es-ES/index/hellox-gpt-4o/}) to generate our textual explanations.} we prompt it with two different strategies.
On the one hand, we employ a prompt that provides the constraints of the \gls{iis} semantically, characterised using natural-language templates, but without any structure. On the other hand, we use a prompt that encodes the graph of reasons as text. In detail, this encoding specifies the graph's nodes, i.e., the constraints of the \gls{iis} semantically characterised with natural language templates (as in the prompt with only the constraints of the \gls{iis}), and the graph's edges, i.e., the relations between the constraints in the \gls{iis}. We show examples of both types of prompts in Tables \ref{tab:prompt:gor} and \ref{tab:prompt:iis} of Appendix \ref{sec:prompt}.

To quantify the benefits of using the graph of reasons as input to the \gls{llm}, we compute two classic \gls{nl} metrics on the generated explanation:
\begin{itemize}
\item \emph{completeness}: the percentage of reasons in the explanation with respect to the number of reasons in the \gls{iis} related to the query;  
and
\item \emph{readability}: difficulty in understanding the text of an explanation. 
\end{itemize}

In general, to compute the readability of an explanation, we use the \gls{gfi} \cite{gunning1952technique}, which estimates the years of formal education a person needs to understand a text. 
Additionally, we count the number of words in an explanation as a measure of its length.

To conduct the above-mentioned readability and completeness analysis, we select a representative sample of different queries both for the \gls{dpp} and the \gls{rcpsp}, such that (i) it contains different types of query constraints, and (ii) the \gls{iis} computed from an \gls{asp} is not the empty set.\footnote{If the \gls{iis} is empty, the \gls{asp} is feasible, which indicates that the user-desired solution is also optimal. In such a case, the explanation states that the proposed optimisation scenario is also optimal; thus, the query can be fulfilled. } 
More specifically, we selected 130 distinct queries for both the \gls{dpp} and the \gls{rcpsp} to compare two types of explanations: those generated using a graph of reasons and those based on an \gls{iis}.

\subsection{Completeness and Readability Results}
\label{subsec:quality}

\begin{table*}[t]
    \caption{Explanations generated using a prompt with the graph of reasons (left) and a prompt with the \gls{iis} (right) to answer the query of the running example: ``\emph{Why is Programming Fundamentals not planned at semester 4?}''. We highlight the reasons in \textit{\textcolor{blue}{blue}}. }
    \centering
    \label{tab:explanations}
    \setlength{\tabcolsep}{5pt}
    \small
    \begin{tabularx}{\textwidth}{p{0.48\textwidth}|p{0.48\textwidth}}
        Programming Fundamentals is not planned in semester 4 because \textit{\textcolor{blue}{it must be taken earlier to satisfy the prerequisite requirements of later subjects}}, particularly \textit{\textcolor{blue}{Programming Techniques and Object Oriented Programming}}. \textit{\textcolor{blue}{These subjects}}, in turn, \textit{\textcolor{blue}{need to be completed before taking Operating Systems Programming}}, \textit{\textcolor{blue}{a compulsory subject}} that \textit{\textcolor{blue}{can only be taken in the autumn}}. Additionally, \textit{\textcolor{blue}{Programming Techniques}} and \textit{\textcolor{blue}{Object-Oriented Programming}} cannot be scheduled in the autumn semesters, further constraining their placement. Since \textit{\textcolor{blue}{subjects cannot be planned more than once, Programming Fundamentals}} must be scheduled early enough to enable this sequence, making semester 4 too late for it. & Programming Fundamentals is scheduled in semester 2 rather than semester 4 to comply with prerequisite requirements and ensure a feasible progression through the curriculum. Since \textit{\textcolor{blue}{this subject must be taken before others like Programming Techniques and Object Oriented Programming}}, delaying it to a later semester would prevent the student from enrolling in these dependent courses on time. Additionally, university rules may restrict when certain subjects can be taken, and \textit{\textcolor{blue}{each subject can only be planned once, so placing Programming Fundamentals}} earlier allows for a coherent and complete academic path within the 6-semester limit.\\
    \end{tabularx}
\end{table*}
\normalsize

Results show that providing the graph of reasons to an \gls{llm} leads to several benefits when generating textual explanations. As an illustrative example for the \gls{dpp}, Table \ref{tab:explanations} shows the explanations generated with and without a graph of reasons. Additionally, Table \ref{tab:explanation_comparison} shows which reasons out of the \gls{iis} are mentioned in each explanation. 
Observe that the explanation using the graph of reasons in the prompt is more complete since it mentions all the reasons, while the one using the \gls{iis} only mentions three (out of eight).

Such an example illustrates a general trend observed in the results presented in Table \ref{tab:explanation_results}, which reports the average completeness, the average \gls{gfi}, and the average number of words of 65 pairs of explanations (one generated using the graph of reasons prompt and the other using the \gls{iis} prompt) for both problem domains.

As for the \gls{dpp}, the average completeness is $84\%$ for explanations generated with a graph of reasons, while the average completeness for explanations generated with an \gls{iis} is $64\%$. This represents a statistically significant gain of $20\%$ when using the graph of reasons as a prompt. 
We also notice that the more complete explanations, thanks to graphs of reasons, are also longer, as shown in Table \ref{tab:explanation_results}. Indeed, explanations generated with a graph of reasons prompt are $32 \%$ longer ($20$ words more) on average than explanations generated with an \gls{iis} prompt. 
Nonetheless, although an explanation generated from the graph of reasons uses, on average, more words, its readability (GFI) is similar to explanations generated using the \gls{iis}.

\begin{table}[t]
    \caption{Reasons mentioned in the explanations (shown in Table \ref{tab:explanations}) generated using a prompt with the Graph of Reasons (GoR) and a prompt with the IIS. }
    \centering
    \label{tab:explanation_comparison}
    \setlength{\tabcolsep}{5pt}
    \begin{tabularx}{\columnwidth}{p{0.75\columnwidth}cc}
        \toprule
        \textbf{Reasons} & \textbf{GoR} & \textbf{IIS} \\
        \midrule
        $r_1$: PF cannot be planned more than once & \textcolor{green}{\checkmark} & \textcolor{green}{\checkmark} \\
        $r_2$: PT cannot be planned at autumn semesters & \textcolor{green}{\checkmark} & \textcolor{red}{\ding{55}} \\
        $r_3$: OOP cannot be planned at autumn semesters & \textcolor{green}{\checkmark} & \textcolor{red}{\ding{55}} \\
        $r_4$: OSP cannot be planned at spring semesters & \textcolor{green}{\checkmark} & \textcolor{red}{\ding{55}} \\
        $r_5$: PF must be planned before PT & \textcolor{green}{\checkmark} & \textcolor{green}{\checkmark} \\
        $r_6$: PF must be planned before OOP & \textcolor{green}{\checkmark} & \textcolor{green}{\checkmark} \\
        $r_7$: PT or OOP must be planned before OSP & \textcolor{green}{\checkmark} & \textcolor{red}{\ding{55}} \\
        $r_8$: OSP is compulsory & \textcolor{green}{\checkmark} & \textcolor{red}{\ding{55}} \\        
        \bottomrule
    \end{tabularx}
\end{table}
\normalsize

\begin{table}[t]
    \caption{Average completeness, number of words, and Gunning Fox Index (GFI) per explanation for both prompt types. }
    \centering
    \label{tab:explanation_results}
    \setlength{\tabcolsep}{2.5pt}
    \begin{tabularx}{\columnwidth}{llccc}
        \toprule
        \textbf{Problem} & \textbf{Metric} & \textbf{Graph Prompt} & \textbf{IIS Prompt} & \textbf{p-value} \\
        \midrule
        \multirow{3}{*}{DPP} & Completeness (\%) & $84\pm3$ & $64\pm4$ & $2.4\cdot10^{-8}$ \\
        & Number of words & $82\pm4$ & $62\pm2$ & $2.1\cdot10^{-12}$ \\
        & GFI & $20.8\pm0.3$ & $20.9\pm0.3$ & $0.86$ \\
        \midrule
        \multirow{3}{*}{RCPSP} & Completeness (\%) & $85\pm 3$ & $69\pm 4$ & $2.1\cdot10^{-7}$ \\
        & Number of words & $107\pm4$ & $85\pm3$ & $9.7\cdot10^{-13}$ \\ 
        & GFI & $15.6\pm0.2$ & $16.5\pm0.3$ & $5.7\cdot10^{-3}$ \\
        \bottomrule
    \end{tabularx}
\end{table}
\normalsize

As to the \gls{rcpsp}, the average completeness is $85\%$ for explanations generated with a graph of reasons prompt, while the average completeness for explanations generated with a \gls{iis} prompt is $69\%$ (therefore, the gain in completeness is $16\%$).
On the other hand, in contrast with the \gls{dpp} case, for the \gls{rcpsp}, the readability index is significantly better for explanations generated with a graph of reasons prompt than explanations generated with an \gls{iis} prompt. Indeed, providing the graph of reasons results in explanations that are easier to understand (\gls{gfi} $=15.6$) with respect to those generated with the \gls{iis} (\gls{gfi} $=16.5$).\footnote{Explanations generated using a graph of reasons contain more concise sentences, which is an indicator directly considered to compute the \gls{gfi}. }

\begin{figure}[t]
        \centering
    \includegraphics[width=\linewidth]{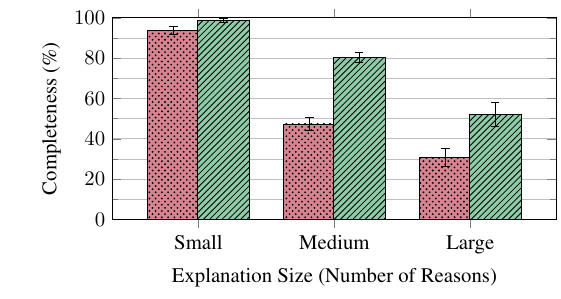}
    \caption{Average percentage of reasons mentioned in the explanations generated using the \gls{iis} (\textcolor{red}{red}) and using the graph of reasons (\textcolor{green}{green}). We group the explanations into three groups depending on the number of reasons: (i) small (fewer than five reasons), (ii) medium (between five and ten reasons), and (iii) large (more than ten reasons).}
    \label{fig:number_of_reasons}
\end{figure}

To better understand how using a graph of reasons affects the completeness of the explanations, Figure \ref{fig:number_of_reasons} shows the aggregated average completeness for both domains. 
We report the average completeness across different explanation groups, categorised by size ---defined as the number of reasons in the graph of reasons (and the corresponding \gls{iis}): (i) small (fewer than five reasons), (ii) medium (between five and ten reasons), and (iii) large (more than ten reasons).
For small explanations (i.e., with fewer than five reasons), the \gls{llm} can generate complete explanations even without a graph of reasons. 
However, for medium and large explanations (i.e., more than five reasons), the completeness drops significantly when the \gls{llm} is given only an \gls{iis}. 
In more detail, for instances with a medium number of reasons (between five and ten reasons), the explanations generated using a graph of reasons mention $\sim 30\%$ more reasons on average than the explanations generated with only an \gls{iis}. Similarly, large explanations ($> 10$ reasons) using a graph of reasons mention $\sim 20\%$ more reasons on average than the explanations generated with an \gls{iis}.

\begin{figure}[t]
        \centering
    \includegraphics[width=\linewidth]{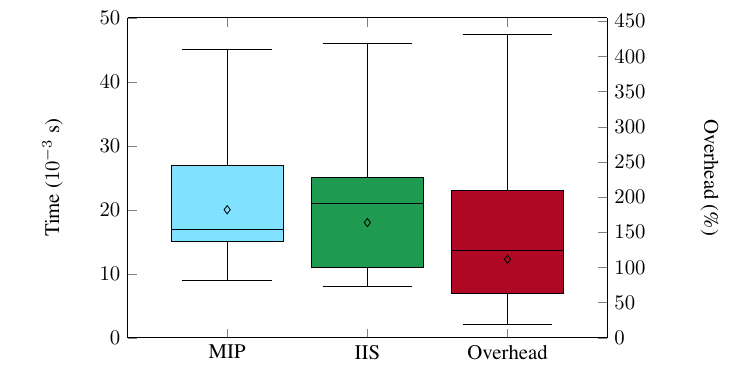}
    \caption{Boxplot of the Solving time of the \gls{mip} problem, the \gls{iis} runtime (left y axis), and the Overhead for the \gls{dpp} (right y axis).}
    \label{fig:boxplot_dpp}
    \end{figure}
\begin{figure}[t]
        \centering
        \includegraphics[width=\linewidth]{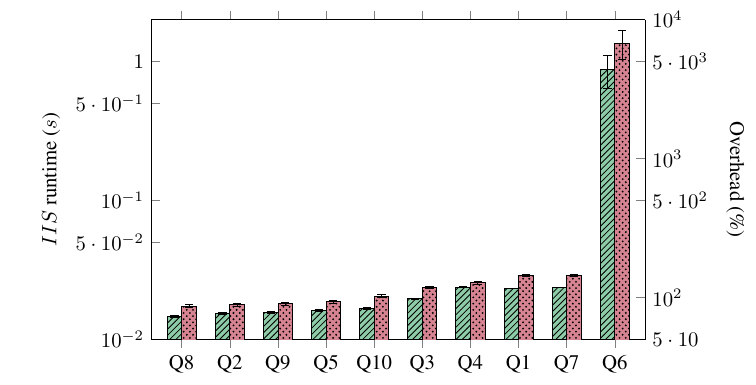}
    \caption{Average \gls{iis} computation time (\textcolor{green}{lines}, left y-axis) and overhead (\textcolor{red}{dots}, right y-axis) for different query types for the \gls{dpp}.}
    \label{fig:barplot_dpp}
\end{figure}

\begin{figure}[t]
        \centering
    \includegraphics[width=\linewidth]{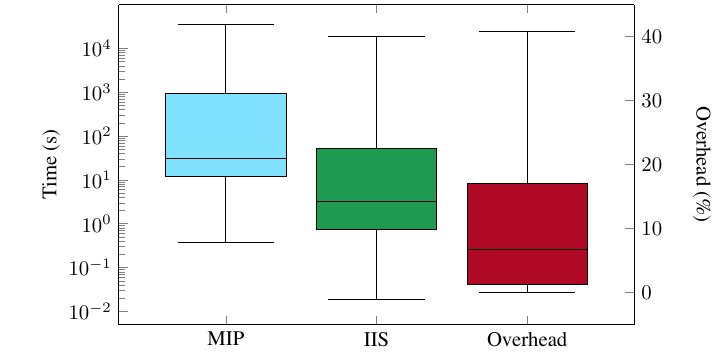}
    \caption{Boxplot of the Solving time of the \gls{mip} problem, the \gls{iis} runtime (left y axis), and the Overhead for the \gls{rcpsp} (right y axis).}
    \label{fig:boxplot_rcpsp}
    \end{figure}
\begin{figure}[t]
        \centering
        \includegraphics[width=\linewidth]{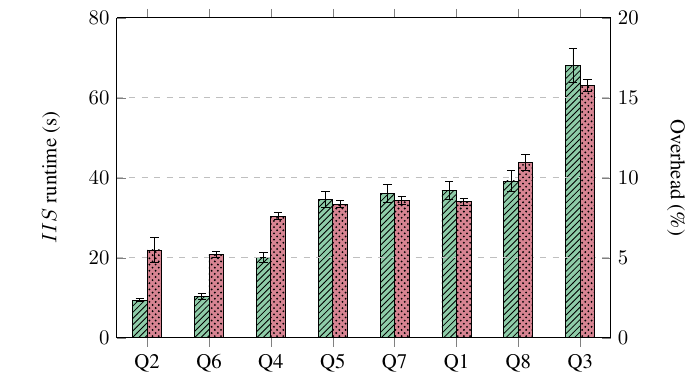}
    \caption{Average \gls{iis} computation time (\textcolor{green}{lines}, left y-axis) and overhead (\textcolor{red}{dots}, right y-axis) for different query types for the \gls{rcpsp}.}
    \label{fig:barplot_rcpsp}
\end{figure}

\subsection{Empirical Hardness Results}
\label{subsec:overhead}

To evaluate the empirical hardness of computing explanations, we will measure (i) the runtimes needed to compute the \glspl{iis} required to build explanations, and (ii) the \emph{time overheads} of computing explanations with respect to solving optimisation problems. 
Thus, given a problem and a query, we define the \emph{time overhead} to build an explanation for the query as the ratio between the time to compute the \gls{iis} and the time to compute the optimal solution for the problem. 
Therefore, our study compares the hardness of solving optimisation problems with that of computing their explanations. 
This allows us to characterise problem instances and compute explanations in real time for users, ensuring a smooth interaction. 

Notice that generating an explanation involves additional steps beyond computing \glspl{iis}, including the construction of a graph of reasons and the formulation of textual explanations using an \gls{llm}. 
However, the cost of constructing the graph of reasons and generating textual explanations is negligible compared to the maximum runtimes required to compute the \glspl{iis}. 
Consequently, we consider only the \gls{iis} runtime when evaluating the complexity of explanation generation.


\paragraph{\gls{dpp}} Figure \ref{fig:boxplot_dpp} shows the \gls{mip} solving times, the \gls{iis} computation times, and the time overhead to compute explanations for \gls{dpp} problem instances. Notice that solving the problem and computing \glspl{iis} is very fast ($\sim 10^{-2}$ seconds) for the majority of cases. Thus, we can provide instant explanations for such a problem. We observe that the average overhead is close to $100\%$, indicating that explanations take the same time to compute as to solve the optimisation problem on average. 
In more detail, only $25\%$ of explanations take more than twice as long to compute as to solve the \gls{mip} problem. 
Figure \ref{fig:barplot_dpp} shows the impact of the query type (see Table \ref{tab:queries_dpp}) on the average runtime to compute an \gls{iis}, and consequently, the required overhead to compute an explanation. On the one hand, most query types are solved on average in approximately a hundredth of a second and have $\sim 100\%$ of time overhead. This result confirms that it is equally computationally costly to compute an explanation as it is to compute the optimal solution.  
On the other hand, computing explanations for query type Q6 is significantly harder than the others. It remains future work to establish why such a query is harder to explain than others. 

\paragraph{\gls{rcpsp}} Similarly to Figure \ref{fig:boxplot_dpp}, Figure \ref{fig:boxplot_rcpsp} shows the \gls{mip} solving times, the \gls{iis} computation times, and the time overhead to compute explanations for \gls{rcpsp} problem instances. 
Notice that solving such \gls{rcpsp} instances is harder than the considered real instance of the \gls{dpp} since harder instances take more than $10^{4}$ seconds to solve. 
However, for such a case, we also notice that computing \glspl{iis} is significantly faster than solving \gls{mip} problems. We observe that, in most cases, the overhead is lower than $40\%$, indicating that most explanations can be computed faster than half the time required to solve a problem instance. 
On the one hand, the median overhead to compute an explanation is $6.7\%$, and the median time is $3.3$ seconds. On the other hand, the median \gls{mip} solving time is $31.2$ seconds. Finally, we compute more than $75\%$ of the explanations in less than 1 minute. Notice that generating fast explanations is a key requirement for users of explainable systems to interact, according to the study by \cite{lakkaraju2022rethinking}. 
Figure \ref{fig:barplot_rcpsp} shows the impact of the query type (see Table \ref{tab:queries_rcpsp}) on the average runtime to compute an \gls{iis}, and the average overhead to compute an explanation. On the one hand, most query types are solved on average in less than 60 seconds and have less than $10\%$ of time overhead. 
On the other hand, some query types are significantly harder than others, e.g., Q3. Similarly to query Q6 for the \gls{dpp}, it remains future work to establish why such a query is harder to explain than others.

\section{Related Work}\label{sec:related}

In the last decade, building explanations to justify the outcomes of ``\emph{black box}'' \gls{ml} models has been mainly the focus of the so-called \gls{xai} community \cite{adadi2018peeking,Mill31122024,Mill31122024realworld,Berendse31122024}, while optimisation (our topic in this paper) has been studied less often in this context. 
Nonetheless, Miller \etal~\cite{miller2019explanation}, in one of the most prominent works of the \gls{xai} literature, identifies useful properties that we consider when building our explanations for optimisation.
According to Miller \etal, a \emph{good} explanation has to be: (i) \emph{contrastive} (in response to particular counterfactual cases); (ii) \emph{social} (to transfer knowledge as a part of a conversation); (iii) \emph{selected} (humans select one or two reasons as an explanation); and (iv) \emph{causal} (explanations with causal relationships are more important than referring to probabilities or statistical relations).
Along these lines, our explanations comply with the properties outlined by Miller \etal{} because: (i) they are \textit{contrastive}, as they explain why user-desired scenarios are not as good as the optimal solution; (ii) they are \textit{social}, as they can be delivered in an interaction dialogue between a user and a \gls{dss}; (iii) they are \textit{selected}, as they are based on an irreducible infeasible set of constraints; and (iv) they are \textit{causal}, since they reference the constraints and the objective function that restrict the output of the optimisation problem.

As mentioned above, the research on providing explanations for optimisation\allowbreak-based \gls{dss} is not as extensive. 
Some of the most commonly used techniques in explainability for optimisation include: (i) what-if analysis \cite{georgara2022building,korikov2021counterfactual}, which involves comparing the optimal solution with a user-desired alternative; (ii) constraint or rule analysis \cite{bogaerts2021framework,junker2004quickxplain,vasileiou2023logic}, which seeks to explain outcomes by reasoning about the problem’s constraints; and (iii) visualisation methods \cite{sartori2023stnweb,walker2020identifying}, which help interpret the performance of the optimisation algorithm. Since we leverage constraint reasoning techniques to provide contrastive explanations, here we discuss what-if and constraint analysis techniques.

Some approaches have been developed to explain the outcome of different classical combinatorial problems by analysing what-if scenarios \cite{korikov2021general,georgara2022building}. However, the explanations generated by such approaches lack some key properties enumerated by Miller \etal~\cite{miller2019explanation}. Although Georgara \etal~\cite{georgara2022building} provide contrastive explanations, they refer to numerical reasons instead of referring to causal relationships. Similarly, Korikov \etal ~\cite{korikov2021counterfactual,korikov2021general} generate counterfactual explanations that indicate the necessary numerical changes in the data of the problem to comply with a user-desired solution. 
Moreover, these approaches are not general; each one addresses a particular application domain. 
More importantly, they do not take advantage of recent advances in \gls{llm} to produce natural language explanations. Instead, they rely on abstract representations~\cite{zehtabi2024contrastive} or fixed templates~\cite{georgara2022building,pozanco2022explaining}, which might not be desirable when providing explanations as part of a conversation with a user~\cite{lakkaraju2022rethinking}. 
In contrast, here we propose \gls{exmip}, a domain-independent explanatory approach that generates contrastive, social, and selected explanations that explain why the optimal solution is better than the one posed by the user.  

Now, there exist approaches focused on constraint or rule analysis that help interpret the outcome of constrained problems. 
For instance, Nardi \etal~\cite{nardi2022graph} find justifications for the outcome of collective decision algorithms via analysing the fulfilled voting rules. 
Bogaerts \etal~
\cite{bogaerts2021framework} give step-wise explanations for \gls{csp} via iteratively analysing the problem constraints. 
Additionally, some works focus on interpreting the infeasibility in Boolean Satisfiability (SAT) problems by extracting a \gls{mus} of clauses \cite{silva2010minimal}, a concept analogous to an \gls{iis} in the context of optimisation problems. 
As a relevant and recent example, Vasileiou et al. \cite{vasileiou2025trace} 
base the production of an explanation for a user query on computing the MUS of a SAT problem. Thereafter, they pass the set of constraints in the explanation to an \gls{llm}, which is tasked with creating a textual explanation. Therefore, Vasileiou \etal~\cite{vasileiou2025trace} directly prompt an \gls{llm} to create a textual explanation without any knowledge about the relationships between the constraints involved in the explanation.
Notice that we evaluated that methodology (i.e., directly providing the \gls{iis} to the \gls{llm}) in our experiment in Figure \ref{fig:number_of_reasons}, showing that our approach (i.e., providing a \emph{structured} representation of the constraints instead of an \emph{unstructured} one, as done in \cite{vasileiou2025trace}) yields a measurable benefit ---especially for medium-large explanations, for which the \gls{llm} struggles to derive the relationships represented by the graph of reasons.


Several other works in the SAT literature discuss focus on techniques to extract multiple \glspl{mus} \cite{liffiton2005finding}, 
since an infeasible system of constraints may possess an exponential number of \glspl{mus} \cite{chakravarti1994some}. 
Along these lines, some methods have been proposed to select the \gls{mus} whose constraints are the most relevant according to user preferences \cite{junker2004quickxplain,gamba2023efficiently}. 
Furthermore, there has been an effort in computing the \gls{mus} with minimal cardinality, i.e., the \gls{mus} with a minimum number of constraints \cite{liffiton2009branch,ignatiev2015smallest}. 
Crucially, the above-mentioned works on \gls{mus} focus on the interpretability of the \emph{infeasibility} of a given \emph{SAT Problem} by extracting a \gls{mus}.
In this paper, we pursue a \emph{different} 
goal. 
We aim to provide \textit{textual explanations} by answering a user's query on the \emph{optimal} solution.
We do so by computing the \gls{iis} of a specially-defined \gls{csp} and identifying the \emph{reasons} that constitute such an explanation, which are then represented as a graph, and then used to build explanations with \glspl{llm}.

\section{Conclusions and Future Work}\label{sec:conclusion}

In this paper, we have tackled the fundamental question of explaining the outcomes of the optimisation processes that drive \acrfull{dss} in many real-world application domains.
Properly explaining these outcomes and answering users' questions is indeed fundamental to ensuring the human oversight required by modern \gls{ai} international regulations.
To achieve this goal, we have shown how we leveraged the recent advances of \glspl{llm} within our approach, namely \gls{exmip}.
However, despite their impressive performance across various language tasks, \glspl{llm} still struggle with abstract reasoning and are prone to producing incorrect and uninformative explanations. To overcome this, within \gls{exmip} we employ constraint reasoning techniques to distil the information necessary to produce the explanation, which is then passed as input to the \gls{llm} as the so-called \emph{graph of reasons}.
Such a graph captures both the semantics and the relationships among the constraints of the optimisation problem. 

We evaluate our approach on two different scheduling problems to showcase its generality and scalability. In detail, we use (i) a real-world instance of the \acrlong{dpp}, based on the Bachelor's Degree in Information and Communications Technology of the Western Sydney University, and (ii) realistic benchmark instances of the well-known \acrlong{rcpsp}. 
Our results demonstrate that providing the graph of reasons to an \gls{llm} significantly improves the quality of the generated textual explanations, both in terms of completeness and readability.
Furthermore, our results indicate that \gls{exmip} is capable of generating explanations in real time in most cases. For the \gls{rcpsp}, our method computes explanations much faster than solving hard instances of the optimisation problem. This represents an improvement over previous approaches in the literature \cite{georgara2022building,zehtabi2024contrastive}, which require re-solving the optimisation problem to generate explanations. 

As future work, we plan to conduct a user study to evaluate whether our explanations are useful to human users. Specifically, we aim to design a study in which bachelor's degree students ---the intended users of the application--- assess the satisfaction, usefulness, and completeness of the explanations provided to justify the recommended degree plans. These criteria are commonly used metrics for evaluating \gls{xai} systems \cite{hoffman2018metrics}. 
In addition, we intend to apply \gls{exmip} to other optimisation-based application domains such as \gls{rpp} \cite{lera2024robust} or value systems aggregation for policy makers \cite{lera2024aggregating}. 

\section*{Data Availability}
The data that support the findings of this study are openly available in Zenodo at \url{https://doi.org/10.5281/zenodo.19665733}, reference number \cite{lera_leri_2026_19665733}. 

\section*{Funding}
Roger X. Lera-Leri, Filippo Bistaffa, and Juan A. Rodríguez-Aguilar were supported by the research project ACISUD (PID2022-136787NB-I00). Athina Georgara is supported by Responsible Ai UK (\url{https://rai.ac.uk/RAiUK}) (EP/Y009800/1).

\section*{Competing Interests}
The authors declare no competing interests.

\bibliographystyle{IEEEtran}
\bibliography{sample}

\appendices

\renewcommand{\thetable}{A\arabic{table}}
\setcounter{table}{0}

\section{IIS example}\label{iis:example}

In Table \ref{tab:iis:example:math}, we provide the mathematical encoding of the constraints within the \gls{iis} of the running example in the left table of Figure \ref{fig:megafig}. Notice that each precedence constraint $c_5$, $c_6$, and $c_7$ is a system of multiple mathematical constraints, instead of a single mathematical constraint. We join them into a single restriction because we need to enforce multiple mathematical constraints per semester to represent the precedence relation of ``\textit{planning a subject after completing its prerequisite}'' (see \gls{dpp} formalisation in \eqref{eq:dpp_mip}).  

\begin{table}[ht]
    \centering
    
    \caption{$IIS$ from Example \ref{example:ASP}. We indicate the constraint type (in different colours) and the corresponding mathematical encoding.}
    \setlength{\tabcolsep}{6pt}
    \def\rowsep{5pt}
    \begin{tabularx}{\columnwidth}{ccc}
        \toprule
        Id & Type & Constraint Encoding \\
        \midrule
        $q$ & Query & $x_{PF,4} = 1$ \\[\rowsep]
        $c_1$ & \textcolor{olive}{Completion} & $\sum_{l=1}^{6} x_{PF,l} \leq 1$  \\[\rowsep]
        $c_2$ & \textcolor{orange}{Offer} & $\sum_{l=1,3,5} x_{PT,l} = 0$  \\[\rowsep]
        $c_3$ & \textcolor{orange}{Offer} & $\sum_{l=1,3,5} x_{OOP,l} = 0$  \\[\rowsep]
        $c_4$ & \textcolor{orange}{Offer} & $\sum_{l=2,4,6} x_{OSP,l} = 0$  \\[\rowsep]
        $c_5$ & \textcolor{blue}{Precedence} & $x_{PF,1} \geq x_{PT,2}, \sum_{l=1}^{3} x_{PF,l} \geq x_{PT,4}$ \\[\rowsep]
        $c_6$ & \textcolor{blue}{Precedence} & $x_{PF,1} \geq x_{OOP,2}, \sum_{l=1}^{3} x_{PF,l} \geq x_{OOP,4}$ \\[\rowsep]
        \multirow{2}{*}{$c_7$} & \multirow{2}{*}{\textcolor{blue}{Precedence}} & $\sum_{l=1}^{2}\sum_{u=\{PT,OOP\}}x_{u,l} \geq x_{OSP,3},$ \\[\rowsep] 
         & & $\sum_{l=1}^{4} \sum_{u=\{PT,OOP\}}x_{u,l} \geq x_{OSP,5}$ \\[\rowsep]
        $c_8$ & \textcolor{red}{Core} & $\sum_{l=1}^{6} x_{OSP,l} = 1$ \\
        \bottomrule
    \end{tabularx}
    
    \label{tab:iis:example:math}
\end{table}
\normalsize

\section{RCPSP Query Translation}\label{subsec:query_translation}

\begin{table*}[t]
    \setlength{\tabcolsep}{2pt}
    \caption{Set of possible user questions, query constraints and constraint encodings for the \gls{rcpsp} formalisation.}
    \label{tab:queries_rcpsp}
    
    \begin{tabularx}{\textwidth}{p{0.025\textwidth}p{0.37\textwidth}p{0.34\textwidth}Y}
        \toprule
        & \textbf{Question} & \textbf{Query constraints} & \textbf{Constraint encoding ($C_Q$)} \\
        \midrule
        Q1.& Why is activity $j$ completed at time $t$? & \textbf{veto} $j$ complete at time $t$ & ${x}_{j,t} = 0$ \\
        Q2.&  Why is activity $j$ \textbf{not} completed at time $t$? & \textbf{enforce} $j$ complete at time $t$ & ${x}_{j,t} = 1$ \\
        Q3.& Why is activity $j$ \textbf{not} completed before time $t$? & \textbf{enforce} $j$ complete before time $t$ & $\sum_{t'=EF_j}^{t-1}{x}_{j,t'} = 1$ \\
        Q4.& Why is activity $j$ \textbf{not} completed after time $t$? & \textbf{enforce} $j$ to complete after $t$ & $\sum_{t'=t+1}^{LF_j}{x}_{j,t'} = 1$ \\
        Q5.& Why is the group of activities $A$ completed at $t$? & \textbf{veto} to complete $j\in A$ at time $t$ & $\sum_{j\in A}{x}_{j,t} < |A|$ \\
        Q6.& Why is the group of activities $A$ \textbf{not} completed at $t$? & \textbf{enforce} to complete all of them at time $t$ & $\sum_{j\in A}{x}_{j,t} = |A|$ \\
        Q7.& Why is activity $j$ completed at $t$ instead of $t'$? & \textbf{enforce} complete $j$ at $t'$ and \textbf{veto} $j$ at $t$& ${x}_{j,t'} = 1$, ${x}_{j,t} = 0$ \\
        Q8.& Why is activity $j$ completed at $t$ instead of $j'$? & \textbf{enforce} $j'$ and \textbf{veto} $j$ complete at time $t$ & ${x}_{j',t} = 1$, ${x}_{j,t} = 0$ \\
        \bottomrule
    \end{tabularx}
\end{table*}
\normalsize

In Table \ref{tab:queries_rcpsp} we show possible user queries tailored to the \gls{rcpsp}. Analogously to the case of the \gls{dpp} in Table \ref{tab:queries_dpp}, we show their encoding to the \gls{mip} formalisation of the \gls{rcpsp}. 

\section{RCPSP Constraint Labels}\label{subsec:rcpsp_constraint_labels}

\begin{table}[ht]
    \caption{Templates for each type of \gls{rcpsp} constraints.}
    \centering
    \label{tab:templates_rcpsp}
    \setlength{\tabcolsep}{3pt}
    \begin{tabularx}{0.9\columnwidth}{ll}
        \toprule
        (i) \textcolor{olive}{Completion}: & ``\emph{Activity j must be completed}'' \\
        (ii) \textcolor{blue}{Precedence}: & ``\emph{Act. h must be completed before Act. j starts}'' \\
        (iii) \textcolor{red}{Resource}: & ``\emph{Res. r is scarce at time $t$ due to Acts. $i,\dots,j$}'' \\
        \bottomrule
    \end{tabularx}
\end{table}

In Table \ref{tab:templates_rcpsp} we show the natural language labels for each type of constraint of the \gls{rcpsp} \gls{mip} formalisation \eqref{prob:RCPSP}. 

\section{LLM Prompt}\label{sec:prompt}

\begin{table}[ht]
    \caption{Prompt using the Graph of Reasons to answer the query of the running example. }
    \centering
    \label{tab:prompt:gor}
    \setlength{\tabcolsep}{5pt}
    \footnotesize
    \begin{tabularx}{\columnwidth}{p{0.23\columnwidth}cccc}
        \multicolumn{5}{p{0.95\columnwidth}}{The bachelor's degree planning problem is the problem of finding a plan of university subjects that respects some precedence and university restrictions while maximising the skills learnt. Specifically, the goal is to acquire the required level for the student's desired job role through the course subjects. In this context, provide an explanation to answer the following question about the solution to the 6-semester bachelor's degree planning problem. To do so, use the reasons below, which are related to the bachelor's degree restrictions. Moreover, these reasons are connected to one another like nodes in a graph. Use the reasons and their connections to write the explanation. 
}  \\
        \vspace{5mm}\\
        \multicolumn{5}{p{0.99\columnwidth}}{Partial Solution: } \\
        
        \toprule
        Subject & Semester & Offer & Predecessors & Core \\
        \midrule
        Prog. Fund. & 2 (Spring) & Both & - & Yes \\
        Prog. Tech. & 4 (Spring) & Spring & PF & Yes \\
        Obj. Or. Prog. & $-$ & Spring & PF & No \\
        Oper. Syst. Prog. & 5 (Autumn) & Autumn & PT or OOP & Yes \\
        \bottomrule
        \vspace{5mm}\\
        \multicolumn{5}{p{0.95\columnwidth}}{Question: ``Why is Programming Fundamentals not planned at semester 4? ''} \\
        \vspace{5mm}\\
        \multicolumn{5}{p{0.95\columnwidth}}{Reasons:} \\ 
        \multicolumn{5}{p{0.95\columnwidth}}{0: Programming Fundamentals is not planned at semester 4. } \\ 
        \multicolumn{5}{p{0.95\columnwidth}}{1: Programming Fundamentals cannot be planned more than once. } \\ 
        \multicolumn{5}{p{0.95\columnwidth}}{2: Programming Fundamentals must be planned before Programming Techniques. } \\
        \multicolumn{5}{p{0.95\columnwidth}}{3: Programming Fundamentals must be planned before Object Oriented Programming. } \\ 
        \multicolumn{5}{p{0.95\columnwidth}}{4: Programming Techniques cannot be planned during autumn semesters. } \\
        \multicolumn{5}{p{0.95\columnwidth}}{5: Programming Techniques or Object Oriented Programming must be planned before Operating Systems Programming. } \\
        \multicolumn{5}{p{0.95\columnwidth}}{6: Object Oriented Programming cannot be planned during autumn semesters } \\
        \multicolumn{5}{p{0.95\columnwidth}}{7: Operating Systems Programming cannot be planned during spring semesters. } \\
        \multicolumn{5}{p{0.95\columnwidth}}{8: Operating Systems Programming is compulsory. } \\ 
        \vspace{5mm} \\
        \multicolumn{5}{p{0.95\columnwidth}}{Connections: } \\
        \multicolumn{5}{p{0.95\columnwidth}}{(0-1), (1-2), (1-3), (2-3), (2-4), (2-5), (3-5), (3-6), (4-5), (5-6), (5-7), (5-8), (7-8)} \\
        \vspace{5mm} \\
        \multicolumn{5}{p{0.95\columnwidth}}{Write a brief explanation without listing reasons and do not refer to the number of the reason. 
} \\ 
    \end{tabularx}
\end{table}
\normalsize

\begin{table}[ht]
    \caption{Prompt using the \gls{iis} (with their constraints translated) to answer the query of the running example. }
    \centering
    \label{tab:prompt:iis}
    \setlength{\tabcolsep}{5pt}
    \footnotesize
    \begin{tabularx}{\columnwidth}{p{0.23\columnwidth}cccc}
        \multicolumn{5}{p{0.95\columnwidth}}{The bachelor's degree planning problem is the problem of finding a plan of university subjects that respects some precedence and university restrictions while maximising the skills learnt. Specifically, the goal is to acquire the required level for the student's desired job role through the course subjects. In this context, provide an explanation to answer the following question about the solution to the 6-semester bachelor's degree planning problem. To do so, use the reasons below, which are related to the bachelor's degree restrictions.}  \\
        \vspace{5mm}\\
        \multicolumn{5}{p{0.99\columnwidth}}{Partial Solution: } \\
        
        \toprule
        Subject & Semester & Offer & Predecessors & Core \\
        \midrule
        Prog. Fund. & 2 (Spring) & Both & - & Yes \\
        Prog. Tech. & 4 (Spring) & Spring & PF & Yes \\
        Obj. Or. Prog. & $-$ & Spring & PF & No \\
        Oper. Syst. Prog. & 5 (Autumn) & Autumn & PT or OOP & Yes \\
        \bottomrule
        \vspace{5mm}\\
        \multicolumn{5}{p{0.95\columnwidth}}{Question: ``Why is Programming Fundamentals not planned at semester 4? ''} \\
        \vspace{5mm}\\
        \multicolumn{5}{p{0.95\columnwidth}}{Reasons:} \\ 
        \multicolumn{5}{p{0.95\columnwidth}}{0: Programming Fundamentals is not planned at semester 4. } \\ 
        \multicolumn{5}{p{0.95\columnwidth}}{1: Programming Fundamentals cannot be planned more than once. } \\ 
        \multicolumn{5}{p{0.95\columnwidth}}{2: Object Oriented Programming cannot be planned during autumn semesters. } \\
        \multicolumn{5}{p{0.95\columnwidth}}{3: Programming Techniques cannot be planned during autumn semesters. } \\ 
        \multicolumn{5}{p{0.95\columnwidth}}{4: Operating Systems Programming cannot be planned during spring semesters. } \\
        \multicolumn{5}{p{0.95\columnwidth}}{5: Programming Fundamentals must be planned before Object Oriented Programming. } \\
        \multicolumn{5}{p{0.95\columnwidth}}{6: Programming Fundamentals must be planned before Programming Techniques. } \\
        \multicolumn{5}{p{0.95\columnwidth}}{7: Programming Techniques or Object Oriented Programming must be planned before Operating Systems Programming. } \\
        \multicolumn{5}{p{0.95\columnwidth}}{8: Operating Systems Programming is compulsory. } \\ 
        \vspace{5mm} \\
        \multicolumn{5}{p{0.95\columnwidth}}{Write a brief explanation without listing reasons, and do not refer to the number of the reason. 
} \\ 
    \end{tabularx}
\end{table}
\normalsize

Now, we provide the prompt templates used to generate the explanations of the experimental evaluation (see Section \ref{sec:exp}). In detail, we provide two prompt templates: (i) a template that uses a graph of reason to build an explanation, and (ii) a template that uses the translation of the mathematical constraints of the \gls{iis}. 

Table \ref{tab:prompt:gor} shows a prompt using the graph of reasons to generate an explanation. Table \ref{tab:prompt:iis} shows a prompt using the translation of the constraints within the \gls{iis} to generate an explanation. Both prompts are used to generate explanations for the running example, shown in Table \ref{tab:explanations}. 
Notice that the only difference between the two prompts is that the prompt using the graph of reasons displays the relations between reasons. Additionally, the order of the reasons is different because, in the case of the graph of reasons, we organise the constraints according to their scope, as mentioned in Section \ref{subsec:ordering}. 
Lastly, in the case of the prompt using the graph of reasons, we add the sentence ``\textit{Moreover, such reasons are connected to each other like nodes in a graph. }'' to provide the necessary context to the \gls{llm}.

\end{document}